\documentclass[preprint,12pt]{elsarticle}
\usepackage{microtype}
\usepackage{booktabs} 
\usepackage{graphicx}
\usepackage{subfigure}
\usepackage{booktabs} 
\usepackage{hyperref}
\usepackage{natbib}
\usepackage{geometry}

\usepackage{amsmath}
\usepackage{amssymb}
\usepackage{mathtools}
\usepackage{amsthm}
\newtheorem{theorem}{Theorem}[section]

\newtheorem{lemma}[theorem]{Lemma}

\usepackage[capitalize,noabbrev]{cleveref}
\usepackage{caption}
\usepackage{colortbl}
\usepackage{xcolor} 
\definecolor{darkergreen}{rgb}{0.0, 0.5, 0.0}

\begin{document}

\begin{frontmatter}

\title{SupReMix: Supervised Contrastive Learning for Medical Imaging Regression with Mixup}

\makeatletter
\def\@fnsymbol#1{}  
\makeatother

\author[csc,tmr]{Yilei Wu\textsuperscript{*}}
\author[csc,tmr,ece]{Zijian Dong\textsuperscript{*}}
\author[duke]{Chongyao Chen}
\author[eth]{Wangchunshu Zhou}
\author[csc,tmr,ece]{Juan Helen Zhou\textsuperscript{†}}

\affiliation[csc]{organization={Centre for Sleep and Cognition \& Centre for Translational Magnetic Resonance Research, Yong Loo Lin School of Medicine, National University of Singapore},
           country={Singapore}}

\affiliation[tmr]{organization={Healthy Longevity \& Human Potential Translational Research Program and Department of Medicine, Yong Loo Lin School of Medicine, National University of Singapore},
           country={Singapore}}

\affiliation[ece]{organization={Department of Electrical and Computer Engineering, National University of Singapore},
           country={Singapore}}

\affiliation[duke]{organization={Department of Mathematics, Duke University},
            city={Durham},
            state={NC},
            country={USA}}

\affiliation[eth]{organization={Department of Computer Science, ETH Zurich},
           city={Zurich},
           country={Switzerland}}

\makeatletter
\renewcommand\@makefntext[1]{\noindent#1}  
\makeatother

\fntext[fn1]{* Equal contribution}
\fntext[fn2]{† Corresponding author}

\begin{abstract}
In medical image analysis, regression plays a critical role in computer-aided diagnosis. It enables quantitative measurements such as age prediction from structural imaging, cardiac function quantification, and molecular measurement from PET scans. While deep learning has shown promise for these tasks, most approaches focus solely on optimizing regression loss or model architecture, neglecting the quality of learned feature representations which are crucial for robust clinical predictions. Directly applying representation learning techniques designed for classification to regression often results in fragmented representations in the latent space, yielding sub-optimal performance. In this paper, we argue that the potential of contrastive learning for medical image regression has been overshadowed due to the neglect of two crucial aspects: \emph{ordinality-awareness} and \emph{hardness}. To address these challenges, we propose \textbf{Sup}ervised Contrastive Learning for Medical Imaging \textbf{Re}gression with \textbf{Mi}xup (\textbf{SupReMix}). It takes anchor-inclusive mixtures (mixup of the anchor and a distinct negative sample) as hard negative pairs and anchor-exclusive mixtures (mixup of two distinct negative samples) as hard positive pairs at the embedding level. This strategy formulates harder contrastive pairs by integrating richer ordinal information. Through theoretical analysis and extensive experiments on six datasets spanning MRI, X-ray, ultrasound, and PET modalities, we demonstrate that SupReMix fosters continuous ordered representations, significantly improving regression performance.
\end{abstract}






\begin{keyword}

Medical imaging regression \sep contrastive learning \sep mixup \sep MRI \sep X-ray \sep ultrasound \sep PET \sep brain age \sep bone age \sep ejection fraction \sep amyloid SUVR



\end{keyword}

\end{frontmatter}



\section{Introduction}

Regression problems aim to predict continuous values based on given input data. They encompass a broad range of medical imaging applications such as predicting brain age using MRI \cite{peng2021accurate,dong2024brain}, assessing pediatric development through bone age prediction from hand X-rays \cite{halabi2019rsna,escobar2019hand}, analyzing cardiac function by measuring left-ventricular ejection fraction (LEVF) from echocardiograms \cite{ouyang2020video}, and evaluating amyloid accumulation in brain PET scans by prediction of standardized uptake value ratios (SUVR) \cite{pemberton2022quantification}. \emph{Vanilla deep regression} refers to the approach of training a deep model to estimate the target value, with the distance (\emph{e.g.}, L1 distance \cite{hsieh2021automated}) between the prediction and ground-truth defined as the loss function.

However, there has been limited research dedicated to developing representation learning methods specifically tailored for regression tasks. Learning robust feature representations, beyond just optimizing regression loss, ensures reliable generalization across different clinical scenarios \cite{chen2019self}. In the realm of classification, techniques such as supervised contrastive learning (SupCon) \citep{khosla2020supervised} have achieved significant success in enhancing representation accuracy. One might consider directly adapting SupCon for regression tasks. However, such direct application of SupCon tends to neglect the inherent ordinal nature of regression, \emph{i.e.}, lack of \emph{ordinality-awareness} (Figure \ref{fig:ordinality}). Furthermore, prior work in classification underscores the significance of hard contrastive pairs in contrastive learning \citep{ho2020contrastive, robinson2020contrastive, kalantidis2020hard, wu2023synthetic}. Nevertheless, they mainly focused on hard negatives with hard positives underexplored, and the \emph{hardness} of contrastive pairs in contrastive learning for regression remains inadequately examined (Figure \ref{fig:logit_figure}). Data mixing techniques \citep{zhang2018mixup,verma2019manifold, shen2022mix} have been used to create hard samples in previous contrastive learning methods for classification \citep{kalantidis2020hard,lee2020mix, liu2023harnessing}. However, these methods do not leverage the label distance to differentiate the hardness among hard negative mixtures.

There have been some notable attempts to address regression problems with constrastive learning. Rank-N-Contrast (RNC) proposes a ranking-based contrastive learning method for regression tasks \cite{zha2023rank}. Adaptive Contrast for Image Regression (AdaCon) introduces an adaptive loss function to preserve label relationships in the latent space \cite{dai2021adaptive}. However, these methods are limited by their reliance on data augmentation, making them not applicable to domains without well-established augmentation techniques, such as time series data.

In this paper, we propose \textbf{Sup}ervised Contrastive Learning for Medical Imaging \textbf{Re}gression with \textbf{Mix}up (\textbf{SupReMix}), a novel framework for regression representation learning. Our aim is to better leverage the inherent ordinal relationships among various inputs and foster the generation of ``harder" contrastive pairs. Instead of relying on real samples in conventional contrastive learning, the proposed SupReMix approach constructs new contrastive pairs at the embedding level in an \emph{anchor-inclusive} and \emph{anchor-exclusive} manner. We take \emph{anchor-inclusive} mixtures as hard negatives: mixing the anchor with a distinct negative sample, thus pulling the negatives closer to the convex hull between the anchor and negatives in order to encourage continuity. On the other hand, we take \emph{anchor-exclusive} mixtures as hard positives: merging two negative samples, the convex combination of whose labels equals to that of the anchor, to encourage local linearity. Moreover, we assign weights to the negative pairs to incorporate label distance information. Our theoretical analysis offers a robust foundation, demonstrating that SupReMix is capable of formulating continuous ordered representations. For the remainder of the paper, we will refer to our hard negatives and hard positives as ``Mix-neg" and ``Mix-pos", respectively.

We validate our approach through extensive experiments on six medical imaging datasets spanning diverse modalities (Figure \ref{fig:enter-label}): brain age prediction from both structural and functional MRI, bone age assessment from X-rays, cardiac function estimation from echocardiograms, and amyloid burden quantification from PET scans. By comparing with other representation learning methods, we demonstrate that SupReMix consistently outperforms existing approaches across all modalities. For example, on the RSNA bone age dataset \citep{halabi2019rsna}, SupReMix achieves state-of-the-art performance with a Mean Absolute Error (MAE) of 4.08 months, significantly improving upon the baseline's 6.79 months. Through both qualitative and quantitative analysis, we validate that SupReMix learns continuous, ordered representations that better capture the inherent structure of regression tasks, leading to more robust and accurate predictions in medical imaging analysis. Furthermore, we show that SupReMix exhibits strong generalization capabilities when handling challenging scenarios such as missing targets and few-shot cases, and can serve as an effective pre-training strategy to enhance existing task-specific architectures.

\section{Related work}

\subsection{Representation learning}

Contrastive learning has emerged as a powerful strategy in self-supervised representation learning, demonstrating improved performance through the alignment of positive pairs and the repulsion of negative pairs in a representation space \citep{chen2020simple, he2020momentum, chen2020improved}. Its supervised variant, termed supervised contrastive learning (SupCon), has been devised as a generalization of triplet \citep{weinberger2009distance} and N-pair losses \citep{sohn2016improved}, wherein pairs of samples from identical classes are considered positive pairs, and those from different classes are seen as negative pairs \citep{khosla2020supervised}. Recently, there have been various adaptations of contrastive learning to continuous labels, with each bringing distinct perspectives to the table. \cite{dufumier2021conditional} leveraged contrastive loss adjusted by continuous metadata for classification, while \cite{schneider2023learnable} utilized a refined contrastive loss to encode behavioral and neural data through interpretable embeddings derived from continuous or temporal labels. Notably, unlike our approach, these studies do not explore regression problems. \cite{yu2021group} devised an action quality assessment model using score regression between two videos, bypassing the usual contrastive learning framework. \cite{wang2022contrastive} added a contrastive loss term to the L1 loss to aid gaze estimation domain adaptation, improving cross-domain performance but reducing source dataset efficiency. In contrast, our method enhances the performance of the source dataset while at the same time facilitating domain adaptation. \cite{zha2023rank} proposed Rank-N-Contrast (RNC), an improved contrastive loss defining positive and negative pairs in a relative way, which refines continuous representations for regression. Adaptive Contrast for Image Regression (AdaCon) preserves
label relationships in the latent space through an adaptive loss function \cite{dai2021adaptive}. However, their reliance on input data augmentation limits their utility in areas devoid of robust augmentation methods, such as time series data. Contrarily, our approach remains applicable to those domains without necessitating augmentation on the input level.

\subsection{Hard contrastive pairs}

Prior research in classification emphasizes the crucial role of hard contrastive pairs in contrastive learning. It generally falls into two categories: hard sample mining and hard sample generation. The former aims to identify the most challenging samples from existing ones, with notable work by \cite{robinson2020contrastive} devising an importance-based sampling strategy for mining hard negatives without computational overhead. In the realm of hard sample generation, \cite{ho2020contrastive} utilized adversarial attacks to augment the training dataset, incorporating pixel-level disturbances into clean samples. \cite{kalantidis2020hard} proposed hard negative mixing strategies at the feature level. \cite{wu2023synthetic} developed a data generation framework enhancing contrastive learning through combined hard sample creation and contrastive learning. These approaches markedly diverge from ours, mainly as they revolve around self-supervised contrastive learning without utilizing label information, and target classification rather than regression. Previous contrastive learning methods for classification have utilized data mixing techniques \citep{zhang2018mixup,verma2019manifold,shen2022mix} to generate hard samples \citep{kalantidis2020hard,lee2020mix,liu2023harnessing}. However, these methods are unable to use label distance to differentiate hardness among hard negative mixtures in regression. C-Mixup \citep{yao2022c} selects samples in regression for mixup based on the label distance. Anchor data augmentation (ADA) \cite{schneider2023anchor} proposes a data augmentation method in nonlinear over-parameterized regression. However, these methods lack a clear definition of the relationship between outcomes (positive or negative) which makes it difficult to define contrastive pairs and explicitly integrate label distance information into the contrastive learning process.

\subsection{Regression in medical imaging}
Medical image regression has emerged as a critical tool in quantitative healthcare, enabling precise measurements across diverse clinical applications. Brain age prediction from structural MRI serves as a biomarker for neurological health, helping identify accelerated aging patterns that may indicate early disease onset \citep{cole2017predicting,smith2019estimation,franke2019ten}. In pediatrics, bone age assessment from hand radiographs aids in evaluating developmental disorders and growth abnormalities \citep{halabi2019rsna,escobar2019hand}. Cardiac function quantification through metrics like ejection fraction from echocardiograms provides essential diagnostic information for heart conditions \citep{ouyang2020video,muhtaseb2022echocotr,mokhtari2022echognn}. In neurology, amyloid PET quantification through standardized uptake value ratios (SUVR) enables early detection and monitoring of Alzheimer's disease (AD) progression \citep{pemberton2022quantification}.

Despite sharing fundamental challenges as regression tasks, these applications have traditionally been approached in isolation, with methods developed specifically for each task separately. For instance, brain age prediction methods often focus on brain-specific architectures \citep{peng2021accurate}, while cardiac function estimation employs specialized temporal modeling \citep{ouyang2020video}. This fragmented approach overlooks the common underlying challenge of learning meaningful representations from medical images for regression tasks. Through extensive experiments across six medical regression tasks, we demonstrate the potential benefits of improving representation learning in medical image regression generally.

\begin{figure}[htbp]
    \centering
    \includegraphics[width=\columnwidth]{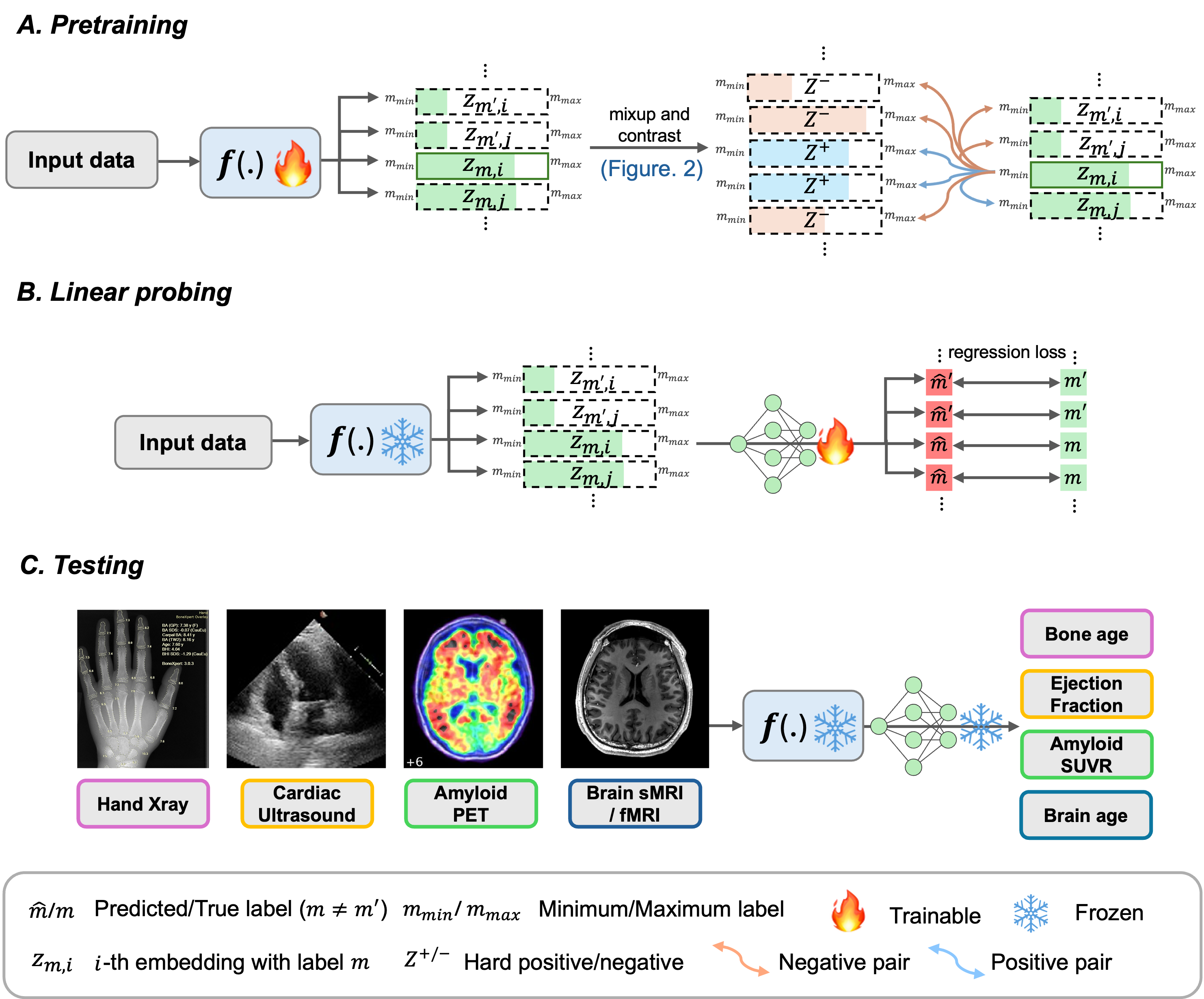}
    \caption{\textbf{Overview of SupReMix framework.} In pretraining (Panel A), the model is trained to learn task-specific representations (\(z_{m,i}\)) through mixup and contrast. In linear probing (Panel B), a linear regressor is trained to predict outcomes such as bone age, ejection fraction, amyloid SUVR, and brain age based on the learned representations. Example input modalities include hand X-ray, cardiac ultrasound, amyloid PET, and brain MRI (Panel C). SupReMix is designed to generalize across diverse medical imaging regression tasks.
    }
    \label{fig:enter-label}
\end{figure}

\section{Method}
\label{sec4}

In a regression task, our goal is to train a neural network that consists of two main components: an encoder $f(\cdot)$: $X \mapsto \mathbf{z} \in\mathbb{R}^{d_e}$ which encodes inputs to embeddings, and a predictor $p(\cdot)$: $\mathbf{z} \in\mathbb{R}^{d_e} \mapsto m\in\mathbb{R}$ which outputs the target value $m \in \mathbb{R}$. Given one mini-batch, hard contrastive pairs are first created utilizing the mixup technique. Following this, we calculate our supervised contrastive regression loss, denoted as $\mathcal{L}_{\text{SupReMix}}$, based on both the real and our hard contrastive pairs. To predict the target value, $f(\cdot)$ is followed by $p(\cdot)$, trained by a regression loss (\emph{e.g.}, L1 loss).

In this section, we outline our approach to supervised contrastive regression. We begin in Section~\ref{sec4.1} with an explanation of our mixup strategy for generating hard negative and positive pairs. This is followed by Section~\ref{sec4.2}, which is about our weights defined for contrastive pairs. This introduces \emph{distance magnifying (DM)}, a behavior that is greatly advantageous in supervised contrastive regression, differentiating it from classification. In Section~\ref{sec4.3}, we bring together the preceding elements to formulate our supervised contrastive regression loss, $\mathcal{L}_{\text{SupReMix}}$. In Section~\ref{sec4.4}, we present a theoretical analysis of the distance magnifying property of our weights for negative pairs, as well as the ordinality-awareness of $\mathcal{L}_{\text{SupReMix}}$.

\textbf{Notations.} Let $I$: the set of embeddings from real samples with $N:=|I|$ (i.e. mini-batch size), $M \subset \mathbb{R}$: the set of all initial labels, ${\rho} :I \mapsto M$: the function mapping an embedding to its label. Order $I$ such that $\rho$ is monotone. $I_{m} \subset I$: the set of embeddings with ${\rho}=m$ and $k_{m}:=|I_{m}|$, which we take to be 0, if $m\notin M$. We give an order for elements in 
$I_{m}$, $(m,i)$, meaning the $i$-th embeddings in $I_{m}$.

\begin{figure}[htbp]
\centerline{\includegraphics[width=\columnwidth]{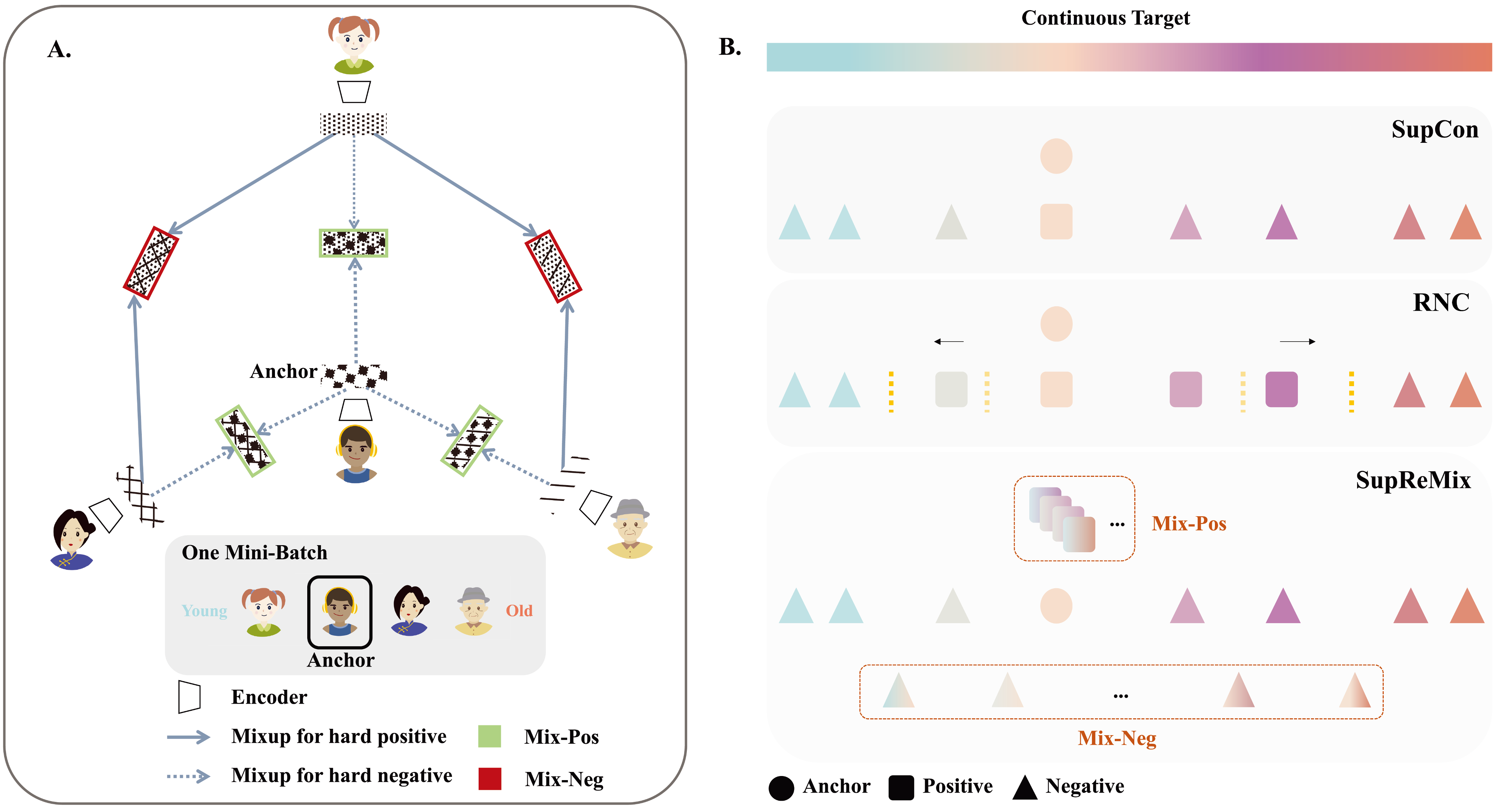}}
\caption{\textbf{Schematic overview of SupReMix method, and comparison with SupCon and RNC.} \textbf{A.} An encoder first encodes inputs to embeddings. Given an anchor, Mix-neg are obtained through mixups ($\lambda_1 \sim \text{Beta}(\alpha, \beta)$) of the anchor itself and a negative in the latent space. Meanwhile, Mix-pos are derived from mixups ($\lambda_2$ is deterministic for two mixup embeddings) of two negative embeddings, the convex combination of whose labels equals to the anchor. \textbf{B.} SupCon identifies samples with the same label as positives and those with different labels as negatives, whereas RNC determines positives and negatives through a relative approach. SupReMix further refines this process by introducing hard positives and negatives alongside the conventional real ones. SupReMix holds a key advantage over RNC: it does not require input augmentation, which can be difficult when dealing with modalities such as time series.}
\label{fig3}
\end{figure}

\subsection{Mixup for hard contrastive pairs}
\label{sec4.1}

\textbf{\emph{Anchor-inclusive} mixtures are hard negatives.} Given an anchor, a ``mixed" negative—created through the convex combination of the anchor itself and a real negative—can be more challenging to differentiate compared to a real negative. This occurs in the latent space where the mixed negative is pulled closer to the anchor, thereby diminishing the discernible differences between the anchor and the hard negative. Given an anchor $\mathbf{z} \color{black}_{m,i} \in I_m$, we generate a set of Mix-neg $\mathbf{z}^{-}_{m,i}$ defined by:

\begin{align}
    \mathbf{z}^{-}_{m,i} &= \lambda_1 \cdot \mathbf{z}_{m,i} + (1-\lambda_1) \cdot \mathbf{z}'
\end{align}

\begin{equation}
    k^{-}_{m,i}:=N-k_{m}
\end{equation}

where $\mathbf{z}' \in I\backslash I_m, \ \lambda_1 \sim \text{Beta}(\alpha, \beta)$, and $\overline{m} = \lambda_1 \cdot m + (1-\lambda_1) \cdot m'$. $k^{-}_{m,i}$ is the number of Mix-neg generated for the anchor $z_{m,i}$, $\overline{m}$ is the label of Mix-neg, which is the convex combination of anchor's label $m$ and a distinct negative's label $m'$.

Different from vanilla mixup \citep{zhang2018mixup}, we use $\lambda_1$ as a ``control" of hardness in our Mix-neg, modulating it through the adjustment of $\alpha$ and $\beta$ parameters that shape the Beta distribution from which $\lambda_1$ is sampled (Figure \ref{fig3}). If we choose $\alpha$ and $\beta$ to produce a skewed distribution where a majority of the values cluster close to one, the anchor $\mathbf{z}_{m,i}$ will almost surely predominate over the real negative $\mathbf{z}'$, thereby generating a harder negative. Conversely, if we choose $\alpha$ and $\beta$ so that $\lambda_1$ is drawn from a distribution that leans heavily towards zero, the anchor $\mathbf{z}_{m,i}$ has a relatively smaller share in the mixup, resulting in a reduction of the mixup hardness.

\textbf{\emph{Anchor-exclusive} mixtures are hard positives.} In contrastive learning, a common practice is to create positive pairs by augmenting an anchor to generate different views \citep{chen2020simple}; despite their visual differences, these augmented input retain their core identity, as they originate from a single input. However, relying solely on this method can limit the richness of the learned representations, as it overlooks the potential value of incorporating other similar objects that offer additional valuable perspectives \citep{wu2023synthetic}. Furthermore, it is not applicable for domains with no appropriate data augmentation methods such as time series. Finally, this approach does not take into account the underlying ordinal relationships among inputs and labels, which is particularly important in regression tasks. 

To address these limitations, we mix two negative embeddings with labels above and below the anchor to serve as hard positives (Figure \ref{fig3}). This strategy not only preserves the natural order of the data but also expands the diversity of the positive pairs, creating a more stringent constraint that guides the learning process to a locally linear embedding space. Given an anchor $\mathbf{z}_{m,i}$, and a window size $\gamma \in \mathbb{Z}^{+}$, we create Mix-pos $\mathbf{z}^{+}_{m,i}$:

\begin{equation}
    \mathbf{z}^{+}_{m,i} = \lambda_2 \cdot \mathbf{z}_{m',i'} + (1-\lambda_2) \cdot \mathbf{z}_{m'',i''}
\end{equation}

\begin{equation}
    k_{m,i}^{+}:=\sum_{j=1}^\gamma k_{i-j}\cdot \sum_{l=1}^\gamma k_{i+l}
\end{equation}

where $i'<i<i'', \ i-i', i''-i \leq \gamma$, and $\lambda_2 \cdot m' + (1-\lambda_2) \cdot m'' = m$. $k_{m,i}^{+}$ is the number of Mix-pos for the anchor $\mathbf{z}_{m,i}$. $0<\lambda_2<1$ is deterministic for each Mix-pos.

\subsection{Distance magnifying}
\label{sec4.2}

In the loss function derived from contrastive pairs, we incorporate a vital parameter - the label distance information for each negative pair (including both real and mixup instances). Leveraging label distance as a metric facilitates a more fundamental approach to handling negative pairs. For each negative pair $\mathbf{z}_{m}$ and $\mathbf{z}_{\overline{m}}$ ($m\neq\overline{m}$), we define a weight $w_{m,\overline{m}}$ as:
\begin{equation}
    w_{m,\overline{m}}=\frac{1+|m-\overline{m}|}{m_{\text{max}}-m_{\text{min}}}
\end{equation}
where $m_{\text{max}}$ and $m_{\text{min}}$ are the maximum and minimum of regression label respectively. (Note that in the implementation, all contrastive pairs will be multiplied by $w$ in the loss calculation. Adding 1 to the numerator ensures that $w$ remains non-zero and the same for all positive pairs). In our loss function (Section~\ref{sec4.3}), the logit corresponding to each contrastive pair is modulated by multiplication with $w$. 

In addition to explicitly encoding label distance information for negative pairs, another critical effect of this weight is to facilitate \emph{distance magnifying (DM)}, a characteristic we argue is highly beneficial in supervised contrastive regression, distinguishing it from classification. Analogously, for a student taking an exam, it is fundamental to answer the easier questions correctly to secure a high score, and basic knowledge harnessed to tackle simpler questions lays the groundwork to address more complex ones. We assert that a similar principle should govern contrastive learning in regression tasks. For example, in the context of brain age prediction from MRI, it is implausible for a model to accurately differentiate brains between ages 30 and 40 if it cannot discern the more pronounced differences between ages 30 and 80.

In Theorem~\ref{th1}, we establish a theoretical analysis that our devised weighting scheme for negative pairs accentuates the influence of the larger label distance. This essentially means that we increase the penalty for negative pairs that are farther apart as compared to those that are closer.

\subsection{Definition of loss function}
\label{sec4.3}

\textbf{Extended Notations.} $\overline{I}$: the new set of training embeddings after mixup ($I \subset \overline{I}$). $\overline{M}\subset \mathbb{R}$: the new set of labels ($M \subset \overline{M}$). Given an anchor $\mathbf{z}_{m,i}$, $I_{(m,i),\overline{m}}$: the set of the anchor's contrastive embeddings with $\rho=\overline{m} \in \overline{M}$, $k_{(m,i),\overline{m}}:=|I_{(m,i),\overline{m}}|$. All the embeddings are normalized to norm 1 in the latent space. Then we formulate our label-wise loss function as:

\begin{equation}
    \mathcal{L}_{\text{SupReMix}} = \sum_{m\in M} \frac{-1}{k_{m}}\sum_{i=1}^{k_m}\sum_{\substack{j=1\\j\neq i}}^{k_{(m,i),m}} \log\frac{e^{\mathbf{z}^T_{m,i}\cdot\mathbf{z}_{m,j}/{\tau}}}{\sum_{\overline{m}\in \overline{M}}\sum_{l=1}^{\prime k_{(m,i),\overline{m}}}w_{m,\overline{m}}\cdot e^{\mathbf{z}^T_{m,i} \cdot\mathbf{z}_{\overline{m},l}/\tau}},
    \label{eq:main_equation}
\end{equation}
where $\sum'$ is the summation running over all $I_{(m,i),\overline m}$ except for $(m,i) = (\overline m,l)$.

In Lemma~\ref{le1}, we establish a lower bound, denoted as \(\mathcal{L}^*\), for \(\mathcal{L}_{\text{SupReMix}}\). Subsequently, in Theorem~\ref{th2}, we demonstrate that this lower bound \(\mathcal{L}^*\) represents the infimum. Furthermore, we conclude that \(\mathcal{L}^*\) can be closed if, and only if, the embeddings adhere to a globally ordered arrangement in accordance to their labels and maintain a locally linear behavior.

\subsection{Theoretical analysis}
\label{sec4.4}

\begin{theorem}[Distance Magnifying]   
\label{th1}
    Given any two negative pairs (real or mixture), $s^{m,m'}_{i,j} := \mathbf{z}^T_{m,i}\cdot  \mathbf{z}_{m',j}$, $s^{m,m''}_{i,l} :=\mathbf{z}^T_{m,i} \cdot  \mathbf{z}_{m'',l}$, where $m\neq m' \neq m''$, $|m'-m|>|m''-m|$, we always have $\nabla_{1}=\frac{\partial \mathcal{L}}{\partial s^{m,m'}_{i,j}}>0$, $\nabla_{2}=\frac{\partial \mathcal{L}}{\partial s^{m,m''}_{i,l}}>0$, $\left.\frac{\nabla_{1}}{\nabla_{2}}\right|_{\mathrm{{with}} \ w}>\left.\frac{\nabla_{1}}{\nabla_{2}}\right|_{\mathrm{{without}} \ w}$ for $\mathcal{L}_{\text{SupReMix}}$. 
\end{theorem}

\begin{lemma}[Lower bound]
\label{le1}
    $\mathcal{L}_{\text{SupReMix}}$ has a lower bound $\mathcal{L}^*$.
\end{lemma}

\begin{theorem}[Infimum]
\label{th2}
    The lower bound $\mathcal{L}^*$ is the infimum of $\mathcal{L}_{\text{SupReMix}}$. $\mathcal{L}_{\text{SupReMix}}$ is closed to its infimum if, and only if, the followings are true:
    \begin{enumerate}
    \addtolength\itemsep{-0.3mm}
     \item all the real samples with the same label $m$ are embedded close to some vector $\mathbf{z}_m$;
     \item all Mix-pos of an anchor $(m,i)$ have embeddings close to $\mathbf{z}_m$;
     \item all the negatives (real and Mix-neg) of an anchor $(m,i)$ have $\mathbf{z}_{m,i}\cdot \mathbf{z}_{m',j}$ that are not equal to $1$.
 \end{enumerate}

 which means that the embeddings are \textbf{globally ordered and locally linear}.
\end{theorem}

\begin{proof}
    Refer to ~\ref{proof}.
\end{proof}

In Theorem \ref{th1}, the key insight is that SupReMix amplifies the penalty for more distant negative pairs compared to closer ones by assigning specific weights to these pairs. Meanwhile, Lemma \ref{le1} and Theorem \ref{th2} collectively highlight that the SupReMix loss function not only possesses a tight bound, but also ensures both the ordinality and continuity of representations as it converges towards the infimum. In ~\ref{proof} we further show that our globally ordered and locally linear representations could lead to a better generalization bound of the regressor with lower Rademacher Complexity \cite{bartlett2002rademacher}.

\section{Experiments}


\subsection{Datasets and experimental setup}

To examine SupReMix's effectiveness across different imaging modalities and regression tasks, we conducted experiments on six different medical imaging datasets spanning brain age prediction (UK Biobank \citep{miller2016multimodal}, HCP-Lifespan \citep{tisdall2012volumetric}), bone age assessment (RSNA-BAA \citep{halabi2019rsna}, RHPE-BAA \citep{escobar2019hand}), cardiac function estimation (Echo-Net \citep{ouyang2020video}), and amyloid burden quantification (A4 \citep{sapra2009anti}). These datasets represent a broad spectrum of medical imaging modalities including structural and functional MRI, X-rays, echocardiogram videos, and PET scans. The key characteristics of each dataset are summarized in Table \ref{tab:datasets}, and we describe each dataset in details below. Additional details about the datasets, including preprocessing steps, are provided in \ref{B1}.

\begin{table}[htbp]
\centering
\resizebox{\textwidth}{!}{
\setlength{\tabcolsep}{4pt}
\begin{tabular}{@{}llllcccc@{}}
\toprule
Dataset & Task & Modality & \#Train & \#Val & \#Test \\
\midrule
UK Biobank & Brain Age & T1 MRI (3D)  & 19509 & 2431 & 2431  \\
HCP-Lifespan & Brain Age & Rs-fMRI (time series) & 456 & 100 & 100 \\
RSNA & Paediatric Bone Age & X-ray (2D) & 11611 & 1000 & 200  \\
RHPE & Paediatric Bone Age & X-ray (2D) & 5496 & 716 & 80  \\
Echo-Net & Ejection Fraction & Echocardiogram (video)  & 7465 & 1288 & 1277  \\
A4 & Amyloid SUVR & PET (3D)  & 3486 & 500 & 500  \\
\bottomrule
\end{tabular}
}
\caption{Dataset characteristics and splits.}
\label{tab:datasets}
\end{table}

\subsubsection{UK Biobank}
UK Biobank \citep{miller2016multimodal} is a large-scale biomedical database and research resource containing in-depth genetic and health information from half a million UK participants. From this rich database, we utilize the T1-weighted structural MRI scans collected from 24,371 participants (aged 42-82 years) to evaluate brain age prediction. The original non-skull-stripped T1-weighted images with 1×1×1 mm³ resolution were resampled to 2×2×2 mm³ and center cropped into 100 x 100 x 100 to reduce computational overhead. Detailed preprocessing steps are described in \cite{ALFAROALMAGRO2018400}.

\subsubsection{HCP-Lifespan}
Human Connectome Project (HCP) Lifespan \citep{tisdall2012volumetric} dataset is part of a comprehensive initiative to map human brain connectivity across the adult lifespan, collecting data from participants aged 36-100 years. From this dataset, we utilize resting-state functional MRI (rs-fMRI) data from 656 participants, acquired using the standard HCP protocol at 2mm isotropic resolution. The fMRI data were preprocessed following the pipeline described in \cite{li2019global}, which includes global signal regression to enhance behavioral associations, and parcellated using the Schaefer-400 atlas \cite{schaefer2018local}. Consequently,
our input data is in the shape of [400 (\# parcells), 478 (\# time frames)].

\subsubsection{RSNA \& RHPE}
For bone age assessment, we utilize two X-ray datasets: the Radiological Society of North America (RSNA) Bone Age Challenge dataset \citep{halabi2019rsna} and the Radiological Hand Posture Estimation (RHPE) dataset \citep{escobar2019hand}. The RSNA dataset contains 12,811 hand radiographs with ages ranging from 1 to 228 months. All radiographs were preprocessed and resized to 520×400 voxels to maintain consistent input dimensions. The RHPE dataset comprises 6,292 hand radiographs annotated with bone ages ranging from 10 to 242 months. For RHPE preprocessing, since each image contains both left and right hands, we first isolated the left hand by taking the left half of each image. Following \cite{escobar2019hand}, we then applied the ground truth bounding boxes to crop the hand region. The cropped images were subsequently resized to 520×400 pixels to maintain consistency with the RSNA dataset preprocessing pipeline. 

\subsubsection{Echo-Net}
Left ventricular ejection fraction (LVEF) is a critical cardiac function metric that measures the percentage of blood leaving the heart with each contraction, serving as a key indicator for diagnosing various heart conditions \cite{ouyang2020video}. The Echo-Net dataset \citep{ouyang2020video} consists of 10,030 apical-4-chamber echocardiogram videos acquired from Stanford University Hospital using various ultrasound machines including iE33, Sonos, Acuson SC2000, Epiq 5G, and Epiq 7C. Each video contains 32 frames and is annotated with LVEF values ranging from 6.91 to 96.96. The dataset has been preprocessed to remove extraneous text information, and all frames were rescaled to 112×112 pixels. During training, frame jitter was applied as an augmentation technique to enhance model robustness.

\subsubsection{A4}
The Anti-Amyloid Treatment in Asymptomatic Alzheimer's (A4) study \citep{sapra2009anti} is a prevention trial in clinically normal older individuals to test whether an anti-amyloid antibody can slow memory loss caused by AD. The dataset consists of 4,448 minimally processed, native-space amyloid PET scans acquired using [18F]-Florbetapir (FBP). Each PET acquisition comprises four five-minute frames collected between 50-70 minutes post-injection, with only these frames included to ensure an adequate signal-to-noise ratio. Each PET scan was generated by averaging co-registered frames. The scans remained in native space without spatial normalization or structural MRI-based processing. Standardized uptake value ratios (SUVRs) were calculated using the whole cerebellum as the reference region across six cortical regions: medial orbital frontal, temporal, parietal, anterior cingulate, posterior cingulate, and precuneus, with values ranging from 0.45 to 2.58. Amyloid positivity (A$\beta$+) was determined based on a predefined SUVR cutoff of $>$ 1.15, following established criteria for amyloid burden classification. We performed quality checks and center-cropped each volume to 128 x 128 x 128.

\subsection{Baseline methods}

We conducted comprehensive comparisons against both representation learning methods and task-specific approaches. For representation learning baselines, we compared SupReMix with classification-based methods including SimCLR \citep{chen2019self} and SupCon \citep{khosla2020supervised}, as well as regression-based approaches including AdaCon \citep{dai2021adaptive} and RNC \citep{zha2023rank}. Data augmentation strategies for baselines were modality-specific: for 2D images (bone age estimation), we employed standard techniques including color distortion and cropping \citep{chen2020simple}; for 3D volumetric data (amyloid SUVR and structural MRI), we utilized rotation, flipping, zooming, and color distortion following \citep{taleb20203d}. For modalities without well-established augmentation methods (fMRI time series), we omitted augmentation during training. Detailed augmentation protocols and a comparison with discretization alternatives are provided in \ref{B1} and \ref{D4}, respectively.

Furthermore, we compared SupReMix with state-of-the-art task-specific methods across different medical imaging domains. For brain age prediction with UK Biobank, we included ResNet-18 \citep{hara2018can}, ViT-B \citep{he2021multi}, Global-Local Transformer \citep{he2021global}, and the fully convolutional SFCN \citep{peng2021accurate}. The HCP-Lifespan time-series analysis incorporated 1D-ResNet \citep{hong2020holmes}, Temporal Convolutional Network (TCN) \citep{lea2017temporal}, and ViT \citep{caro2023brainlm}. Cardiac function estimation on Echo-Net was evaluated against EchoNet \citep{ouyang2020video}, graph neural network-based EchoGNN \citep{mokhtari2022echognn}, and transformer-based EchoCotr \citep{muhtaseb2022echocotr}. For bone age assessment (RSNA-BAA and RHPE-BAA), we compared with BoNet \citep{escobar2019hand}, PEAR-Net \citep{liu2020self}, Doctor Imitator (DI) \citep{chen2021doctor}, SIMBA \citep{gonzalez2020simba}, and multi-scale MMANet \citep{yang2023multi}. For Amyloid PET quantification (A4), we evaluated against SFCN \citep{peng2021accurate}, 3D DenseNet \citep{huang2017densely}, and attention-based FiA-Net \citep{he2021multi}. All methods were implemented following their original training protocols. 

\subsection{Implementation details}
For RSNA and RHPE datasets (X-ray, 2D image), we adopted a ResNet-50 \citep{he_resnet} as the backbone network. For UK Biobank and A4 (T1/PET MRI, 3D volumetric images), we adopted a 3D-ResNet-18 \citep{hara2018can} as the backbone network. For Echo-Net (Echocardiogram, video), we used a r2plus1d-18 network as the backbone\cite{tran2018closer}. For HCP-Lifespan, we adopt an 18-layer 1D-ResNet \citep{hong2020holmes} as the backbone network to process time-series data. 

Across all experiments, we utilized Adam optimizer \citep{kingma2014adam} with an initial learning rate of $10^{-3}$. For representation learning frameworks, all methods were pretrained for 200 epochs for a fair comparison, followed by linear probing for 100 epochs. Temperature parameter ($\tau$) selections are dataset-dependent, with values of 0.5 for UK Biobank and HCP-Lifespan and 1.0 for RSNA, RHPE, Echo-Net and A4. When applying a vanilla regression approach, our implementation follows the guidelines established in previous studies \citep{pmlr-v139-yang21m,peng2021accurate,cheng2023weakly}. Complete implementation details are in  \ref{C1}.

\subsection{Evaluation metrics}
We use common metrics for regression \citep{pmlr-v139-yang21m} to evaluate the performance, including Mean-Absolute-Error (MAE), Mean-Square-Error (MSE), Geometric Mean (GM) error, and Pearson correlation coefficient ($\rho$).

\section{Results}

\subsection{Understanding Supervised Contrastive Learning for Regression}

\subsubsection{Ordinality-awareness}

\begin{figure}[h]
    \centering
    \includegraphics[width=\columnwidth]{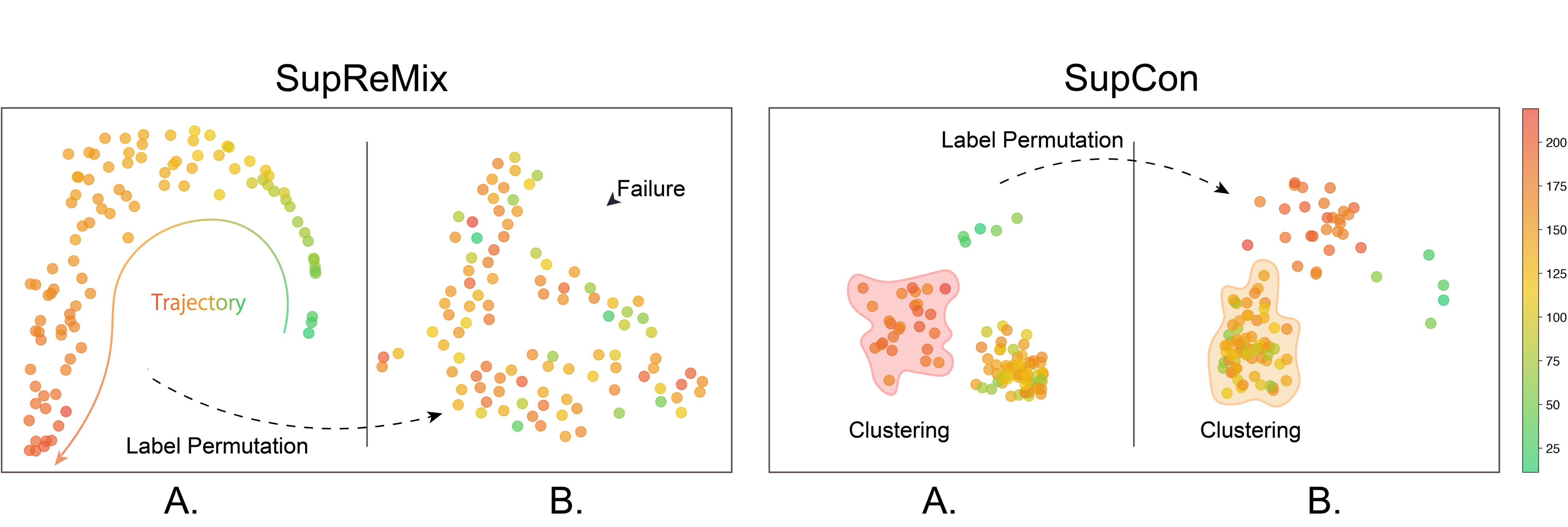}
    \caption{\textbf{Visualization (2D t-SNE map \cite{JMLR:v9:vandermaaten08a}) of learned representations from RSNA dataset \citep{halabi2019rsna} with genuine and permuted bone age labels.} \textbf{A}: representations from genuine labels; \textbf{B}: representations from permuted labels. Our method produces continuous and ordered representations in the latent space that would be disrupted if the labels were permuted. In contrast, classification-based methods like SupCon create clusters regardless of whether the labels are genuine or permuted.}
    \label{fig:ordinality}
\end{figure}

In regression tasks, a well-constructed model should learn representations where the ordering of data points in the latent space reflects the ordering of their corresponding labels \cite{zha2023rank}. To evaluate this property, we performed a label permutation experiment on RSNA X-rays \citep{halabi2019rsna} for bone age prediction. Our label permutation protocol randomly reassigns a new value to each bone age group in the dataset. For example, all X-rays originally labeled as 12 months would be reassigned to the same new value (\emph{e.g.}, 48 months), while X-rays from a different age group would be collectively reassigned to another value (\emph{e.g.}, 12 months). While this permutation preserves the equivalence relationships within each bone age group, it deliberately disrupts the natural ordering between groups. A model that truly captures the ordinal nature of regression data should exhibit markedly different behavior when trained with permuted versus genuine labels.

Shown in Figure~\ref{fig:ordinality}, when trained with genuine labels, SupReMix produces embeddings that form a continuous, ordered structure reflecting the natural progression of bone age (left panel-A). However, this structure is significantly disrupted with permuted labels (left panel-B), indicating SupReMix's sensitivity to label relationships. Conversely, SupCon treats different bone ages as discrete classes, producing clustered embeddings that remain largely unchanged between genuine and permuted labels (right panel). This reveals a fundamental limitation of applying classification-based contrastive learning to regression tasks -- such approaches fail to leverage the ordinal information inherent in regression labels.

\subsubsection{Hardness}

\begin{figure}[h]
    \centering
    \includegraphics[width=\columnwidth]{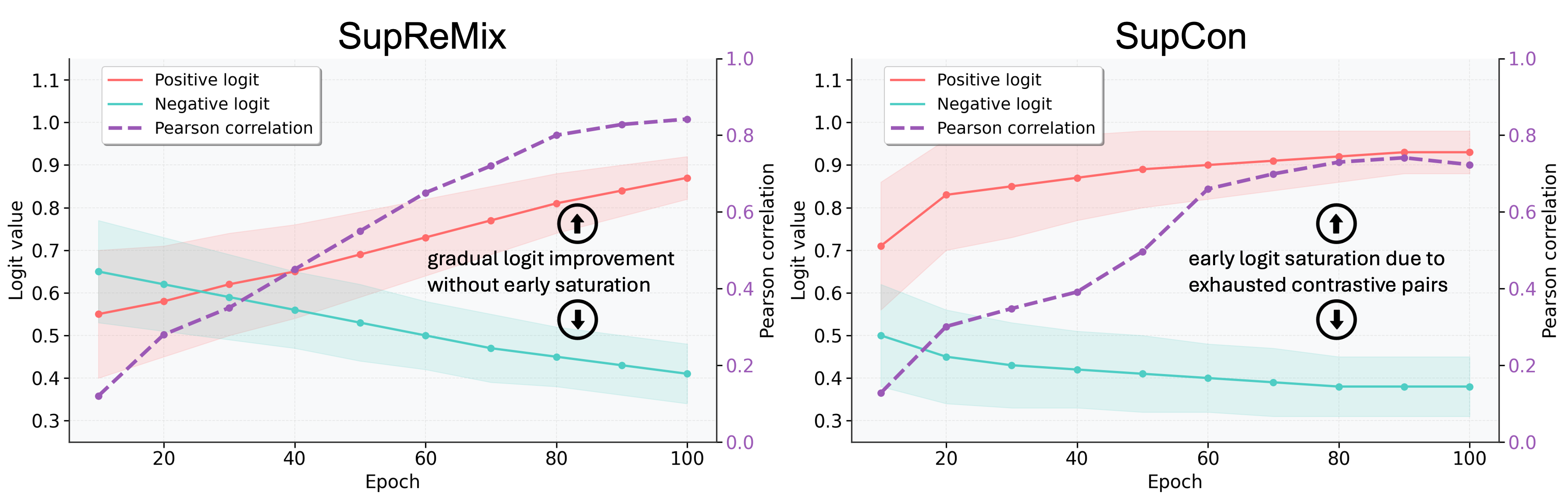}
    \caption{\textbf{Comparison of average logits over training (epochs).} SupCon exhibits early logit saturation due to exhausted contrastive pairs, while SupReMix maintains gradual improvement without early saturation. The logit values (left y-axis) and Pearson correlation (right y-axis) are tracked over 100 training epochs.}
    \label{fig:logit_figure}
\end{figure}

Previous studies on contrastive learning have demonstrated that as the training process advances, a diminishing number of hard negative samples make contributions to the overall loss function \citep{kalantidis2020hard}. Our observations confirm this phenomenon and extend it to positive samples as well, particularly when applied to regression problems. In Figure \ref{fig:logit_figure}, we utilized the RSNA bone age dataset to examine the average of logit values (\emph{i.e.}, the inner product of two embeddings with norm of 1) representing the similarity between two embeddings over the training process. We monitor all positive pairs and the top 1k most challenging negative pairs. We note that, in SupCon, the logit values for positive pairs saturate swiftly, approaching 1 after around 100 training epochs. This indicates a reduced reliance on these pairs in later learning stages as they contribute less to the loss function. Similarly, fewer negative pairs contribute to the loss function over time. In contrast, SupReMix continued to improve its performance throughout 100 epochs, due to its more challenging contrastive pairs.

\subsection{Comparison with other representation learning methods}
We compare SupReMix against both classification-based (SimCLR \citep{chen2019self}, SupCon \citep{khosla2020supervised}) and regression-based (AdaCon \citep{dai2021adaptive}, RnC \citep{zha2023rank}) representation learning methods across six medical imaging regression datasets spanning different modalities (MRI, X-ray, ultrasound, and PET). As shown in Figure \ref{fig:mae_comparision}, SupReMix consistently achieves lower MAE than all baseline methods. Taking bone age prediction on RSNA as an example, SupReMix reduces the MAE to 4.08 months compared to the classification-based method SupCon (6.79 months). The performance gains are particularly notable in challenging tasks such as cardiac function estimation, where SupReMix achieves an MAE of 4.02 compared to SupCon's 5.88. Moreover, we observe that SupReMix consistently outperforms vanilla regression (shown in dashed line). Additional performance metrics including Mean Squared Error (MSE), correlation coefficient ($\rho$), and Geometric Mean (GM) error are provided in Tables \ref{table1}, \ref{table2}, \ref{table3}, \ref{table4}, \ref{table5}, \ref{table6}.

\begin{figure}[htbp]
    \centering
    \includegraphics[width=\columnwidth]{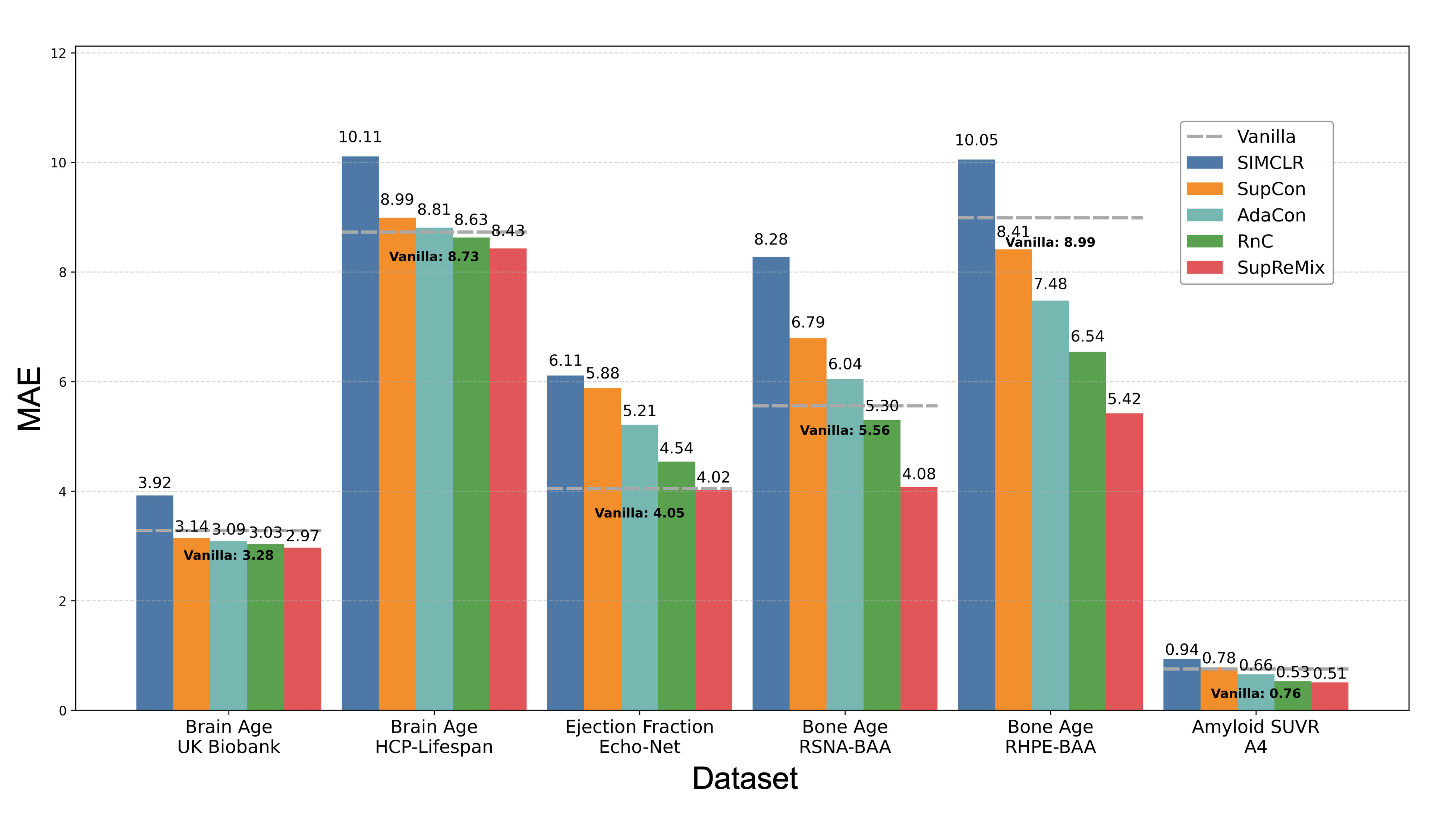}
    \caption{
    \textbf{Mean Absolute Error (MAE) Comparisons Across Datasets.} 
    The figure above compares the MAE of six methods: Vanilla, SIMCLR, SupCon, AdaCon, RnC, and SupReMix, across six datasets: UK Biobank, HCP-Lifespan, Echo-Net, RSNA, RHPE, and A4. Statistical significance between methods, determined by paired t-tests, is annotated, where *** indicates \( p < 0.001 \), ** indicates \( p < 0.01 \), and ns denotes no significant difference. SupReMix demonstrates the lowest MAE across most datasets, showcasing superior performance on broad medical imaging regression tasks.}
    \label{fig:mae_comparision}
\end{figure}

To better understand the relationship between prediction performance and distribution of learned latent representations, we analyze different methods through both scatter plots (showing prediction accuracy) and t-SNE \citep{van2008visualizing} visualizations in two tasks: brain age prediction from UK Biobank MRI (Figure \ref{fig:scatter_tsne}, shown in blue) and amyloid burden (SUVR) prediction from A4 PET scans (Figure \ref{fig:scatter_tsne}, shown in red). As illustrated in Figure~\ref{fig:scatter_tsne}, the visualization reveals distinct characteristics and limitations of existing approaches. SimCLR,\citep{van2008visualizing} while maintaining reasonable prediction accuracy, produces scattered representations that lack a clear ordinal structure in the latent space. Classification-based methods such as SupCon tend to form discrete clusters, which is suboptimal for continuous regression tasks, as evidenced by the disconnected groups in its t-SNE plot. RnC shows improved ordinal awareness but still struggles to fully capture the continuous nature of the underlying data, particularly at the extremes of the value range. In contrast, SupReMix demonstrates superior representation learning performance by preserving both local smoothness and global ordering, as shown by the more coherent progression in the t-SNE space and tighter alignment along the perfect prediction line in the scatter plots. The visualization aligns with the quantitative improvements observed in our empirical evaluation.

\begin{table}[h!]
    \centering
    \begin{tabular}{cc}
        \begin{minipage}[bt]{0.47\textwidth}
        \centering
        \caption{Evaluation on UKB }
        \resizebox{0.97\textwidth}{!}{%
                \begin{tabular}{lcc}
                    \toprule
                    Metrics & $\rho {\uparrow}$ & MSE ${\downarrow}$\\
                    \midrule
                    Vanilla  & 0.822 & 16.12 \\
                    SimCLR & 0.712 & 19.12 \\
                    SupCon & 0.844 & 15.78 \\
                    AdaCon & 0.850 & 15.24 \\
                    RNC & 0.856 & 14.69 \\
                    \rowcolor{gray!20} \textbf{SupReMix} & \textbf{0.863} & \textbf{14.37} \\
                    \midrule
                    \rowcolor{white} GAINS (\textbf{Ours} VS. Vanilla(\%)) & \textbf{\textcolor{darkergreen}{+5.0}} & \textbf{\textcolor{darkergreen}{+10.9}} \\
                    \bottomrule
                \end{tabular}
                \label{table1}
            }
        \end{minipage}
        &
        \begin{minipage}[bt]{0.47\textwidth}
        \centering
        \caption{Evaluation on HCP }
        \resizebox{0.97\textwidth}{!}{
            \begin{tabular}{llc}
                \toprule
                Metrics  &$\rho {\uparrow}$& MSE ${\downarrow}$ \\
                \midrule
                Vanilla   &0.822& 123.22 \\
                SimCLR  &0.712& 153.49 \\
                SupCon  &0.844& 126.43 \\
                AdaCon  &0.847& 124.11 \\
                RNC  &0.856& 121.78 \\
                \rowcolor{gray!20}\textbf{SupReMix}  &\textbf{0.863}& \textbf{116.11} \\
                \midrule
                \rowcolor{white} GAINS (\textbf{Ours} VS. Vanilla(\%)) &\textbf{\textcolor{darkergreen}{
                $+$5.0}}& \textbf{\textcolor{darkergreen}{
                $+$5.8}} \\
                \bottomrule
            \end{tabular}
            \label{table2}}
        \end{minipage} 
        \\ \\
        \begin{minipage}[bt]{0.47\textwidth}
            \centering
             \caption{Evaluation on EchoNet }
            \resizebox{0.97\textwidth}{!}
            {%
                \begin{tabular}{lcc}
                    \toprule
                    Metrics & MSE ${\downarrow}$ & GM ${\downarrow}$\\
                    \midrule
                    Vanilla  & 20.55 & 4.51 \\
                    SimCLR & 50.12 & 5.88 \\
                    SupCon & 40.32 & 5.39 \\
                    AdaCon & 32.94 & 4.86 \\
                    RNC & 25.55 & 4.32 \\
                    \rowcolor{gray!20} \textbf{SupReMix} & \textbf{19.79} & \textbf{3.80} \\
                    \midrule
                    \rowcolor{white} GAINS (\textbf{Ours} VS. Vanilla(\%)) & \textbf{\textcolor{darkergreen}{+3.7}} & \textbf{\textcolor{darkergreen}{+15.7}} \\
                    \bottomrule
                \end{tabular}
                \label{table3}
            }
        \end{minipage} 
        &
        \begin{minipage}[bt]{0.47\textwidth}
            \centering
            \caption{Evaluation on RSNA-BAA }
            \resizebox{0.97\textwidth}{!}{%
                \begin{tabular}{lcc}
                    \toprule
                    Metrics & MSE ${\downarrow}$ & GM ${\downarrow}$\\
                    \midrule
                    Vanilla  & 74.231 & 4.449 \\
                    SimCLR & 117.388 & 5.418 \\
                    SupCon & 76.523 & 4.887 \\
                    AdaCon & 73.198 & 4.455 \\
                    RNC & 69.873 & 4.023 \\
                    \rowcolor{gray!20} \textbf{SupReMix} & \textbf{43.652} & \textbf{3.221} \\
                    \midrule
                    \rowcolor{white} GAINS (\textbf{Ours} VS. Vanilla(\%)) & \textbf{\textcolor{darkergreen}{+41.2}} & \textbf{\textcolor{darkergreen}{+27.6}} \\
                    \bottomrule
                \end{tabular}
                \label{table4}
            }
        \end{minipage} 
        \\ \\
        \begin{minipage}[bt]{0.47\textwidth}
            \centering
            \caption{Evaluation on RHPE-BAA }
            \resizebox{0.97\textwidth}{!}{%
                \begin{tabular}{lcc}
                    \toprule
                    Metrics & MSE ${\downarrow}$ & GM ${\downarrow}$\\
                    \midrule
                    Vanilla  & 142.763 & 5.892 \\
                    SimCLR & 182.784 & 6.311 \\
                    SupCon & 124.425 & 5.358 \\
                    AdaCon & 99.409 & 5.145 \\
                    RNC & 74.392 & 4.932 \\
                    \rowcolor{gray!20} \textbf{SupReMix} & \textbf{62.983} & \textbf{3.992} \\
                    \midrule
                    \rowcolor{white} GAINS (\textbf{Ours} VS. Vanilla(\%)) & \textbf{\textcolor{darkergreen}{+55.9}} & \textbf{\textcolor{darkergreen}{+32.2}} \\
                    \bottomrule
                \end{tabular}
                \label{table5}
            }
        \end{minipage} 
        &
        \begin{minipage}[bt]{0.47\textwidth}
            \centering
            \caption{Evaluation on A4 }
            \resizebox{0.97\textwidth}{!}
            {%
                \begin{tabular}{llc}
                    \toprule
                    Metrics  &$\rho {\uparrow}$& MSE ${\downarrow}$ \\
                    \midrule
                    Vanilla   &0.755& 0.006 \\
                    SimCLR  &0.542& 0.018 \\
                    SupCon  &0.712& 0.007 \\
                    AdaCon  &0.755& 0.006 \\
                    RNC  &0.764& 0.005 \\
                    \rowcolor{gray!20} \textbf{SupReMix}  &\textbf{0.792}& \textbf{0.004} \\
                    \midrule
                    \rowcolor{white} GAINS (\textbf{Ours} VS. Vanilla(\%))  &\textbf{\textcolor{darkergreen}{+4.9}}& \textbf{\textcolor{darkergreen}{+33.3}} \\
                    \bottomrule
                \end{tabular}
                \label{table6}
            }
        \end{minipage} 
    \end{tabular}
    \label{table:main_result}
\end{table}
\begin{figure}[htbp]
    \centering
    \includegraphics[width=\columnwidth]{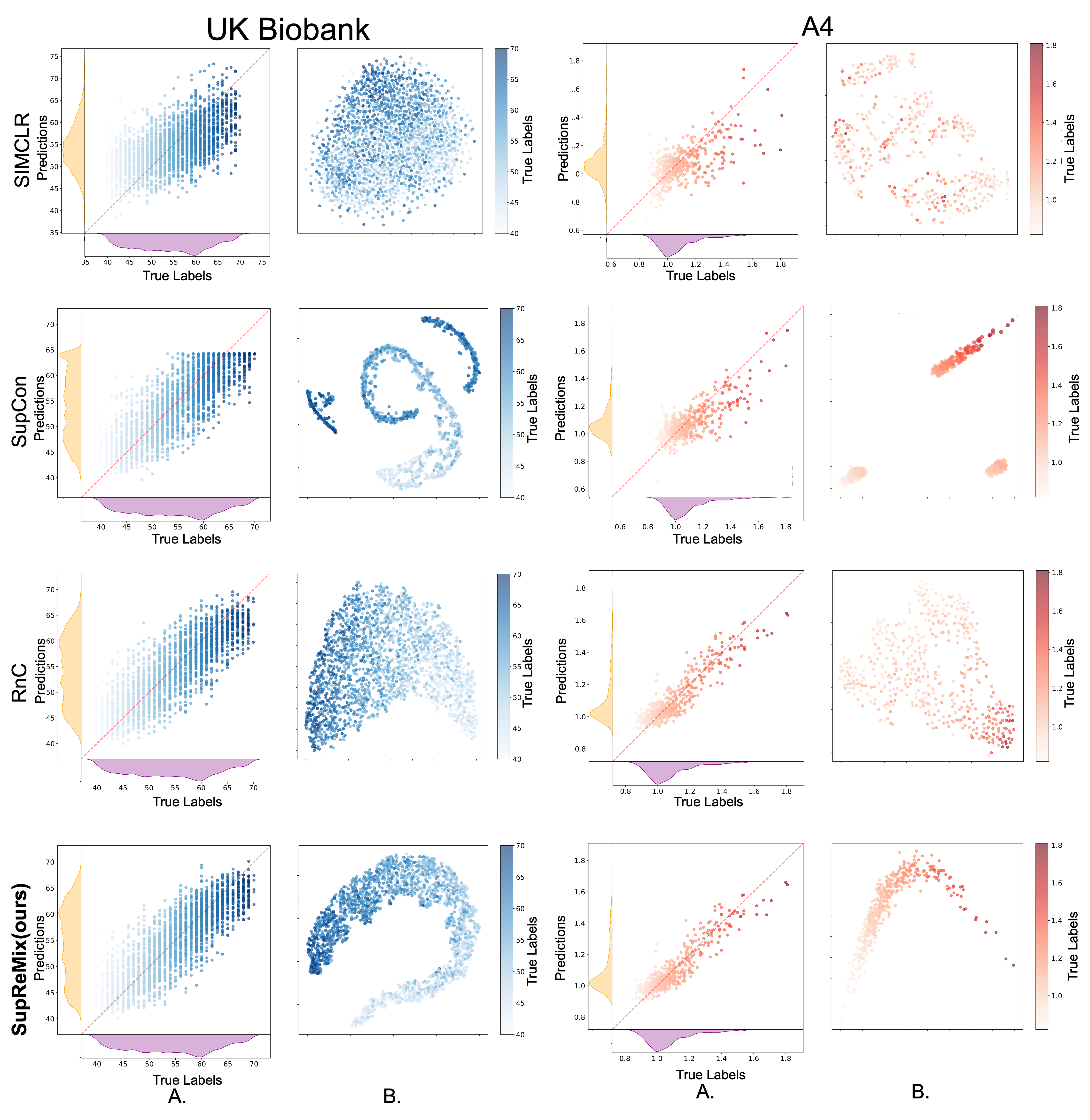}
    \caption{\textbf{Visualization of learned representations on two medical imaging tasks: brain age prediction (blue) and amyloid burden prediction (red).} \textbf{A}: Scatter plots of true vs. predicted values with density. \textbf{B}: t-SNE visualizations of learned representations colored by true labels. Baseline methods show various limitations: SimCLR lacks ordinal structure, SupCon forms discrete clusters unsuitable for continuous regression, and RnC shows incomplete ordinal preservation. In contrast, SupReMix leads to better representations: (1) smooth progression in t-SNE visualization for both age and amyloid burden predictions, and (2) effectively maintaining both local structure and global ordering, leading to better prediction accuracy across the full value range.}
    \label{fig:scatter_tsne}
\end{figure}

\newpage

\subsection{Representation continuity}
We have qualitatively observed the continuity of representations through t-SNE plot in Figure \ref{fig:scatter_tsne} . In this section, we quantitatively analyze their continuity and smoothness by adapting the Lipschitz continuity analysis \citep{tang2024understanding}. Specifically, for a representation $\phi$, we examine the local Lipschitz factor ($L$) between neighboring points:

\begin{equation}
    L(x, x') = \frac{\|T(x) - T(x')\|}{\|\phi(x) - \phi(x')\|}
\end{equation}

where $x$ and $x'$ are two input data, and $T(\cdot)$ represents the regression target. We normalize $L$ to account for different embedding dimensions by constructing a Normalized Lipschitz Factor Distribution (NLFD):
\begin{itemize}
    \item First, we apply full-batch normalization to ensure zero mean and unit variance per coordinate across the dataset $\mathcal{D}$.
    \item For each point $x \in \mathcal{D}$, we compute the Lipschitz factor with respect to its nearest $\ell_2$ neighbor in representation space.
    \item For each point $x \in \mathcal{D}$, we compute the Lipschitz factor with respect to its nearest $\ell_2$ neighbor in representation space.
    \item To standardize across different embedding dimensions $d$, we scale all factors by $\sqrt{d}$.
\end{itemize}

The resulting distribution's skewness provides crucial insights into representation quality. Left-skewed distributions (←) are better because they indicate that the representation is dominated by smaller Lipschitz factors, meaning the representation preserves local similarities more effectively. As shown in Figure \ref{fig:NLFD} (A), SupReMix consistently produces more left-skewed distributions across different medical imaging tasks, indicating superior  structures of continuity. This skewness towards smaller Lipschitz factors suggests that SupReMix learns representations where similar inputs reliably map to similar positions in latent space - a crucial property for robust regression.

To quantify the relationship between representation continuity and regression performance, we conduct a bootstrapping analysis with bootstrap sample size $B=100$, computing both the NLFD gap and regression performance gap. To formally quantify the differences between NLFDs from SupReMix ($\phi_\text{supremix}$) and vanilla ($\phi_\text{vanilla}$) representations, we calculate a Z-score gap:
$Z = \frac{\mu_{\phi_\text{vanilla}} - \mu_{\phi_\text{supremix}}}{\sqrt{\sigma^2_{\phi_\text{vanilla}} + \sigma^2_{\phi_\text{supremix}}}}$
where $\mu_\phi$ and $\sigma_\phi$ represent the mean and standard deviation of the NLFD for a given representation $\phi$. From Figure \ref{fig:NLFD} (B), we observed consistent correlations across datasets (UK Biobank: $\rho=0.57$, RSNA: $\rho=0.69$, A4: $\rho=0.55$), empirically validating that smoother representations lead to better regression performance.

\begin{figure}[htbp]
    \centering
    \includegraphics[width=\columnwidth]{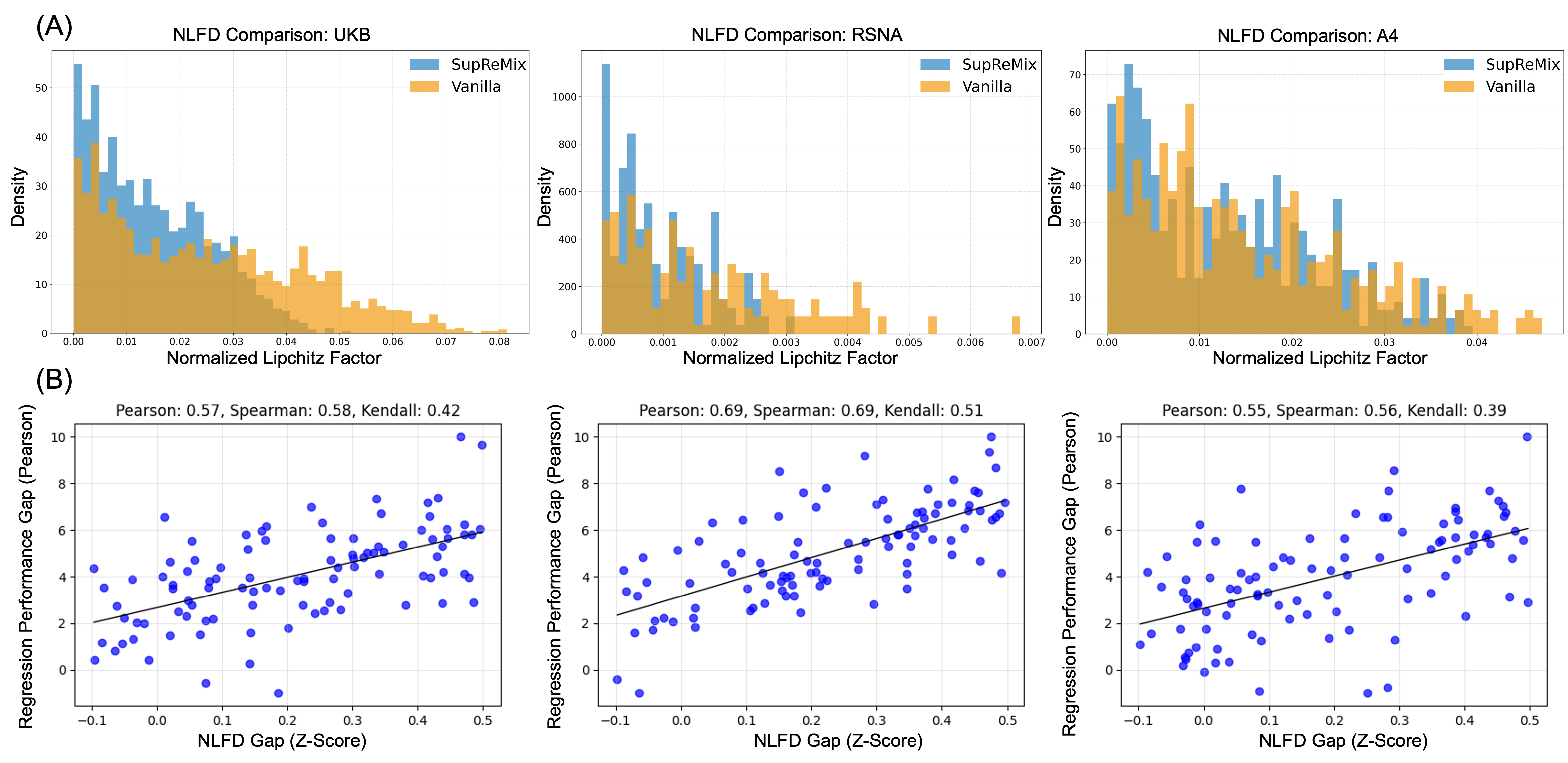}
    \caption{\textbf{Analyzing representation continuity in medical image regression tasks. }
    \textbf{(A)} Left-skewness (←) is better. Normalized Lipschitz Factor Distributions (NLFD) comparing SupReMix (blue) and vanilla (orange) representations on UK Biobank brain MRI , RSNA pediatric hand X-ray, and A4 Amyloid PET datasets. 
    \textbf{(B)} Scatter plots showing the relationship between representation smoothness (x-axis, measured by NLFD gap) and regression performance gap (y-axis, measured by correlation gap) across bootstrapped samples. Pearson correlations are reported for each task (UK Biobank: $\rho=0.69$, RSNA: $\rho=0.55$, A4: $\rho=0.57$). Black lines indicate linear regression fits.}

    \label{fig:NLFD}
\end{figure}

\subsection{SupReMix pretraining improves task-specific methods}

For each medical imaging regression task, researchers have proposed various specialized architectures and loss functions to address task/modality-specific challenges. For instance, in brain age prediction, SFCN \citep{peng2021accurate} introduces a fully convolutional architecture with careful kernel size selection to capture age-related brain patterns, while Global-Local Transformer \citep{he2021global} leverages both local anatomical features and global brain structure through a dual-branch design. To understand whether SupReMix pretraining can improve on top of these task-specific methods, we compare two configurations for each architecture (Figure \ref{fig:multi_panel_mae_comparison}): direct training with task-specific methods (blue bars) and training combined with SupReMix pretraining (orange bars), with SupReMix serving as a \emph{plug-and-play} solution that requires no modifications to the core architecture designs. 

As shown in the Figure \ref{fig:multi_panel_mae_comparison}, for brain age prediction on UK Biobank, we evaluate four methods including the widely-used ResNet-18 \citep{hara2018can}, the vision transformer ViT-B \citep{he2021multi}, the Global-Local Transformer \citep{he2021global} designed for capturing multi-scale features, and SFCN \citep{peng2021accurate} specialized for brain age estimation. Across all these methods, SupReMix pretraining improves the MAE by \textbf{4.5\%-17.3\%} compared to training from scratch. Similarly, for HCP-Lifespan, we examine three architectures including 1D-ResNet \citep{hong2020holmes}, temporal convolutional network (TCN) \citep{lea2017temporal}, and ViT \citep{caro2023brainlm} for sequential feature extraction, where SupReMix leads to consistent improvements of \textbf{4.3\%-7.1\%}. In Echo-Net, where we evaluate methods for echocardiography analysis, we compare architectures including EchoNet \citep{ouyang2020video}, EchoGNN for graph-based architecture \citep{mokhtari2022echognn}, and EchoCotr \citep{muhtaseb2022echocotr} with a CNN-transformer architecture, observing MAE reductions of \textbf{0.7\%-14.2\%}. For RSNA-BAA and RHPE-BAA (bone age assessment), we examine specialized architectures such as BoNet \citep{escobar2019hand}, PEAR-Net \citep{liu2020self}, DI \citep{chen2021doctor}, SIMBA \citep{gonzalez2020simba}, and MMANet \citep{yang2023multi}, achieving improvements \textbf{up to 7.4\% and 27.6\%} , respectively. In Amyloid-PET analysis, we evaluate SFCN \citep{peng2021accurate}, 3D DenseNet \citep{huang2017densely}, and FiA-Net \citep{he2021multi}, with improvements ranging \textbf{from 2.4\% to 10.6\%}. 

These systematic comparisons demonstrate that SupReMix pretraining can consistently achieve performance improvements across a diverse range of medical imaging regression tasks regardless of modalities and underlying architectures (CNNs, transformers, or graph neural networks) and their task-specific optimizations.

\begin{figure}[htbp]
    \centering
    \includegraphics[width=\columnwidth]{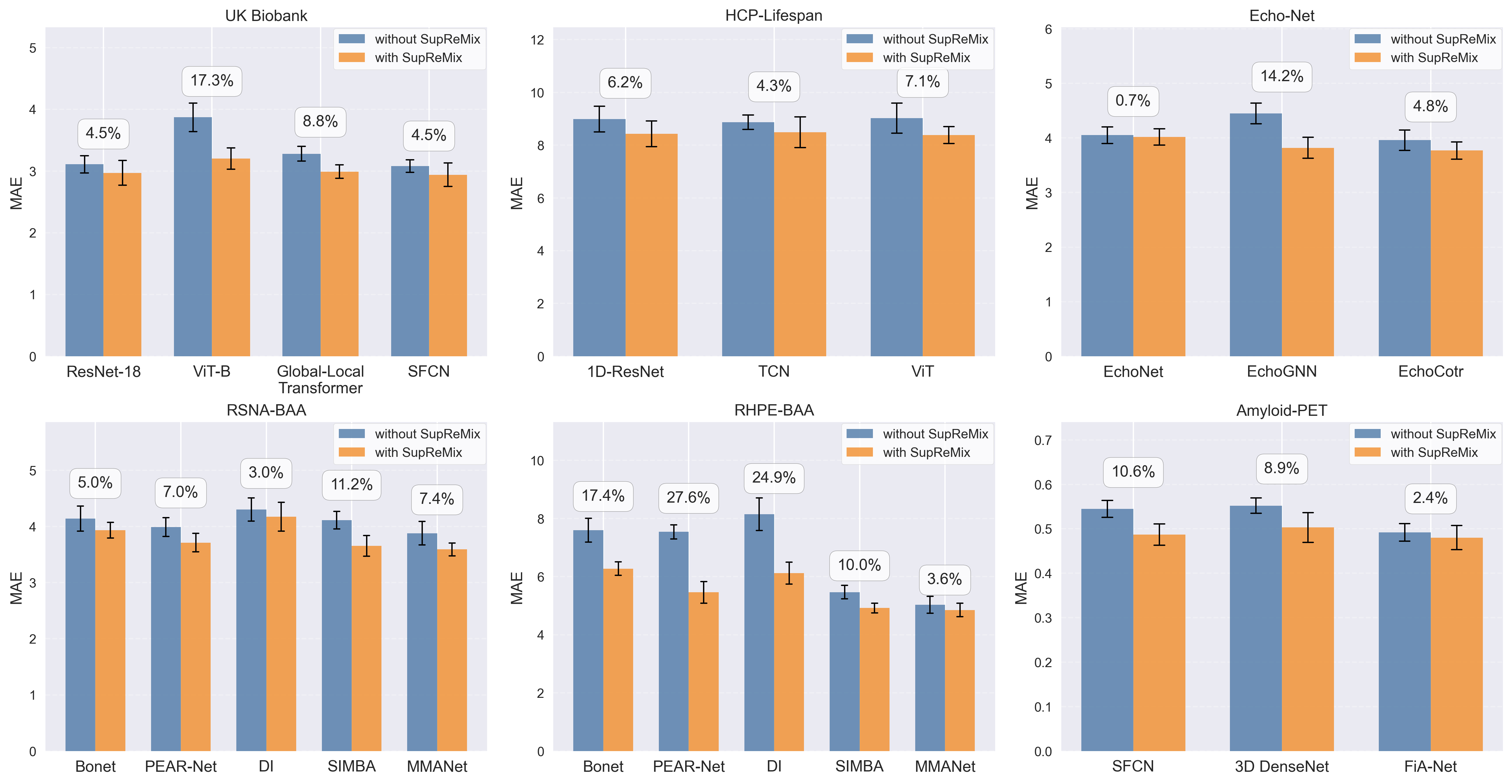}
    \caption{\textbf{Mean Absolute Error (MAE) comparisons between task-specific methods with and without SupReMix pretraining.} Each panel represents a specific dataset, showcasing the performance improvement (in percentages) achieved by integrating SupReMix into various models. The height of each bar represents the mean value calculated from three runs with different random seeds, while the error bars indicate the corresponding standard deviations.}
    \label{fig:multi_panel_mae_comparison}
\end{figure}
\newpage

\subsection{Resilience to reduced training data}
While modern deep learning has achieved remarkable success through large-scale training datasets, obtaining extensive labeled medical images with precise continuous annotations (\emph{e.g.}, age, disease scores, organ measurements) remains a significant challenge. This challenge stems from the high cost of expert annotations and time-intensive clinical assessments. This is particularly true for regression tasks in medical imaging, where acquiring accurate numerical labels often requires specialized expertise and complex clinical measurements \citep{zhou2023deep}.

To address this limitation, developing methods that can perform robustly under limited training data scenarios is crucial for practical medical image analysis applications. In this study, we conduct a comprehensive analysis to investigate how SupReMix influences model performance under different data availability scenarios across five medical imaging datasets.

As shown in Figure~\ref{fig:data_efficiency}, we evaluate the performance by varying the training set sizes. Our experiments span diverse datasets with different scales: UK Biobank (2,500-20,000 samples), A4 Amyloid-PET (500-3,500 samples), RSNA-BAA (2,000-12,000 samples), RHPE-BAA (1,000-5,000 samples), and EchoNet (1,000-7,000 samples). For each dataset and sample size configuration, we compare two settings while keeping all other training parameters consistent: vanilla training (blue solid lines) and SupReMix (orange dashed lines).

The results demonstrate SupReMix's significant advantages, particularly in low-data regimes. This is evidenced by consistent performance gains across all three metrics (MAE, R, and MSE) when training data is limited. In the extreme low-data scenario of UK Biobank (2,500 samples), SupReMix achieves an MAE of 3.8 years compared to 4.7 years for vanilla training - a 19\% reduction in error. The performance advantage is even more pronounced in the A4 dataset, where SupReMix maintains R $>$ 0.7 with just 500 samples, while vanilla training requires approximately 2,000 samples (4x more data) to achieve comparable performance. Similar patterns emerge in RSNA-BAA, where SupReMix achieves R = 0.62 with 2,000 samples, matching vanilla training performance at 6,000 samples.

\begin{figure}[h!]
    \centering
    \includegraphics[width=\columnwidth]{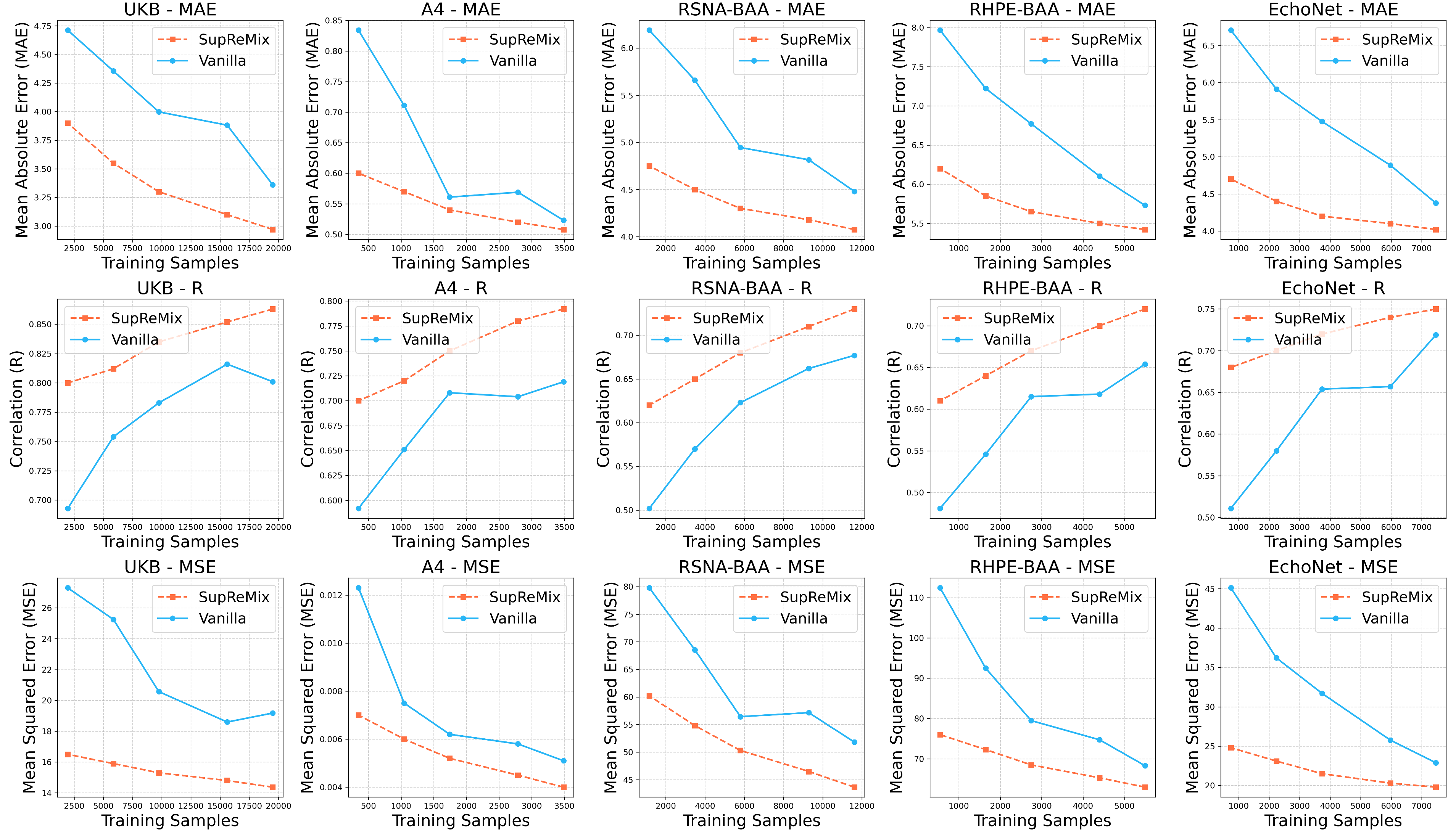}
    \caption{\textbf{Performance comparison of SupReMix and Vanilla models with reduced training data.} The solid blue lines represent SupReMix, while the dashed orange lines represent Vanilla models. Each row represents a specific metric: MAE (Mean Absolute Error), R (Correlation), and MSE (Mean Squared Error). Each column corresponds to a dataset. SupReMix demonstrates consistent performance improvements across all datasets and metrics, particularly with smaller training sample sizes, highlighting its robustness in low-data regimes. }
    \label{fig:data_efficiency}
\end{figure}

\newpage

\subsection{Transfer learning}

\begin{table}[h]
    \centering
    \footnotesize
    \caption{Transfer learning results comparing different methods across bone age assessment datasets.}
    \begin{tabular}{l|cccc|cccc}
    \toprule
    & \multicolumn{4}{c|}{RSNA $\rightarrow$ RHPE (subsampled, 2k)} & \multicolumn{4}{c}{RHPE $\rightarrow$ RSNA (subsampled, 1k)} \\
    \cmidrule{2-9}
    & \multicolumn{2}{c}{Linear Probing} & \multicolumn{2}{c|}{Fine-tuning} & \multicolumn{2}{c}{Linear Probing} & \multicolumn{2}{c}{Fine-tuning} \\
    \cmidrule{2-9}
    Method & MAE$^{\downarrow}$ & R$^{\uparrow}$ & MAE$^{\downarrow}$ & R$^{\uparrow}$ & MAE$^{\downarrow}$ & R$^{\uparrow}$ & MAE$^{\downarrow}$ & R$^{\uparrow}$ \\
    \midrule
    SimCLR & 11.85 & 0.512 & 9.42 & 0.658 & 11.36 & 0.521 & 9.12 & 0.668 \\
    SupCon & 11.43 & 0.528 & 9.12 & 0.672 & 11.02 & 0.538 & 8.89 & 0.682 \\
    AdaCon & 10.86 & 0.545 & 8.75 & 0.688 & 10.85 & 0.545 & 8.65 & 0.695 \\
    RnC & 10.24 & 0.562 & 8.32 & 0.702 & 10.54 & 0.562 & 8.42 & 0.712 \\
    \textbf{SupReMix} & \textbf{9.58} & \textbf{0.589} & \textbf{7.85} & \textbf{0.728} & \textbf{10.12} & \textbf{0.589} & \textbf{8.04} & \textbf{0.735} \\
    \bottomrule
    \end{tabular}
    \label{table:transfer}
\end{table}

We evaluate the transferability of representations learned by SupReMix across different bone age assessment datasets. Our evaluation focuses on two transfer learning scenarios: (1) from RSNA (source, ~12K samples) to RHPE (target, ~2K subsampled samples), and (2) from RHPE (source, ~6K samples) to RSNA (target, 2K subsampled samples). For each scenario, we assess two transfer learning strategies: linear probing, where we train only a linear layer while keeping the pre-trained encoder frozen, and fine-tuning, where we update the entire network on the target dataset.

As shown in Table~\ref{table:transfer}, SupReMix consistently outperforms all baseline methods across both transferring scenarios and different strategies. In the RSNA~$\rightarrow$RHPE transferring, SupReMix achieves MAEs of 9.58 and 7.85 for linear probing and fine-tuning,  respectively, representing substantial improvements of 19.2\% and 16.7\% over SimCLR. The performance advantage persists in the RHPE$\rightarrow$~RSNA direction, where SupReMix attains MAEs of 10.12 (linear probing) and 8.04 (fine-tuning). In addition, the strong correlation coefficients achieved by SupReMix in all settings further validate its ability to learn robust and transferable representations for bone age assessment tasks. These results demonstrate that SupReMix not only excels in single-dataset scenarios but also facilitates effective knowledge transferring across different bone age assessment datasets.

\subsection{Gender-Aware Representation in Bone Age Assessment}

Another key characteristic of SupReMix is its ability to preserve important biological factors when forming contrastive pairs. In the RSNA dataset, our approach generates hard-negative and hard-positive pairs exclusively within the same gender groups, which directly contributes to the gender-aware representations shown in Figure~\ref{fig:RSNA-gender}. The visualization reveals distinct continuous trajectories for male and female subjects, both following clear age progression patterns.

By constraining mixup operations within gender boundaries, SupReMix effectively models the natural gender differences in bone maturation—where females typically mature earlier than males of the same chronological age. This approach highlights the importance of biologically informed representation learning in medical imaging regression tasks.

\begin{figure}[htbp]
    \centering
    \includegraphics[width=0.8\columnwidth]{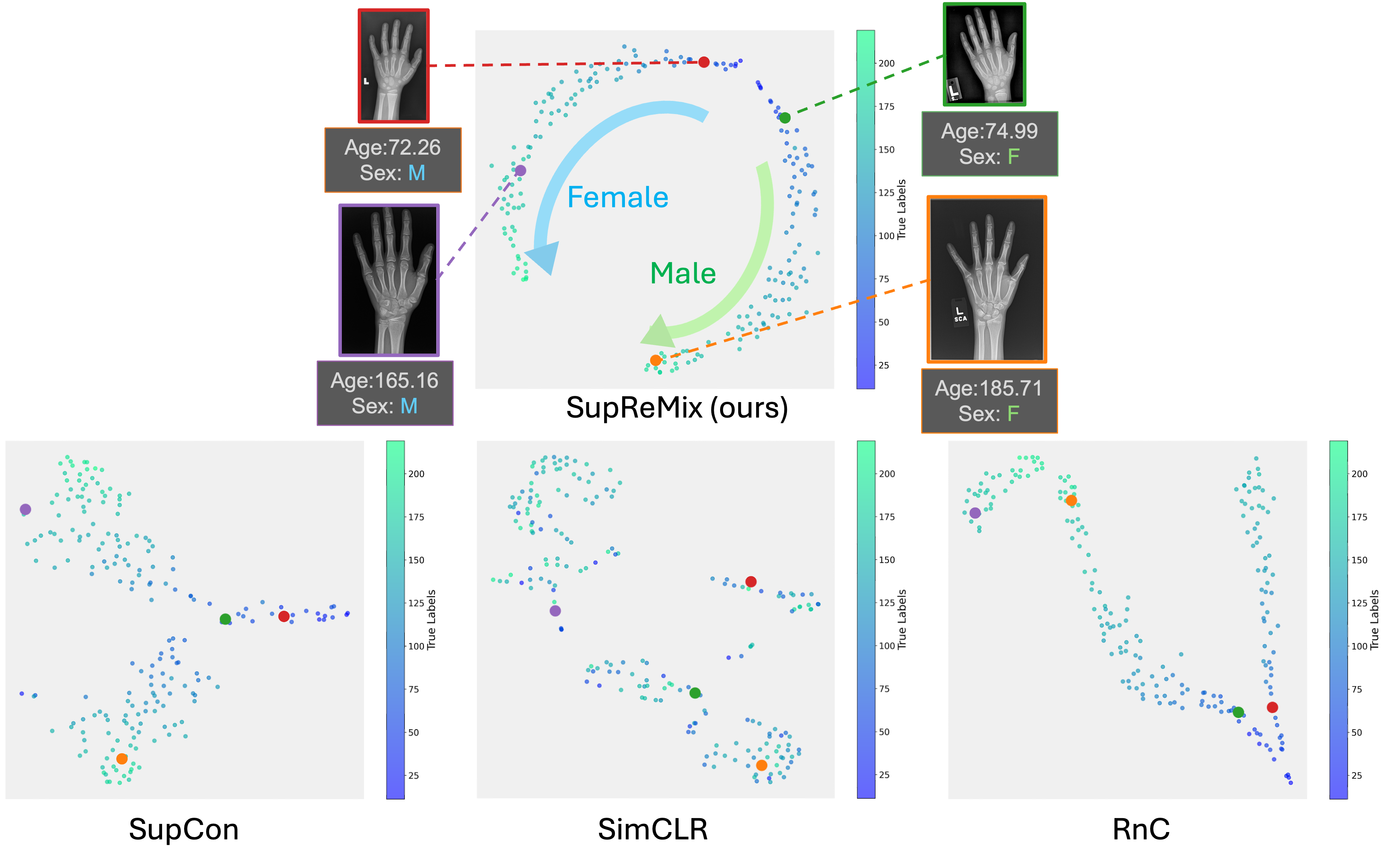}
    \caption{\textbf{SupReMix simultaneously captures bone age progression and gender differences in the latent space.} While SupReMix (top) forms continuous gender-separated trajectories, alternative methods (bottom) struggle to preserve both gender differentiation and age continuity.}
    \label{fig:RSNA-gender}
\end{figure}

\subsection{Ablation studies}

We conduct ablation studies on RSNA-BAA dataset to investigate each component in SupReMix. As shown in Table \ref{table9}, each component contributes to performance improvement, with the complete SupReMix framework achieving the best results. We also examine the impact of window size $\gamma$ for Mix-pos pair generation. Table~\ref{table10} shows that moderate window sizes ($\gamma=5$ for RSNA, $\gamma=1.0$ for UK Biobank) achieve optimal performance, while too small or too large windows lead to degraded results.

Finally, we study three beta distributions for sampling the mixing ratio $\lambda_1$ in Mix-neg generation, each representing different sampling strategies relative to the anchor point. The Beta(2,8) distribution represents a ``close-to-negative" sampling strategy, where the generated mix-neg samples are biased towards the negative samples, creating more moderate hard negative pairs. In contrast, the Beta(8,2) distribution represents a ``close-to-anchor" strategy where mix-neg samples are closer to the anchor, resulting in more aggressive hard negatives. The Beta(5,5) distribution provides a symmetric sampling strategy with no bias towards either anchor or negative samples. As shown in Table \ref{table11}, the ``close-to-negative" strategy (Beta(2,8)) consistently outperforms other configurations across all six datasets, with notable improvements on UK Biobank (MAE/GM: 2.98/14.37) and Echo-Net (4.88/6.30). This suggests that generating mix-neg samples closer to the negative examples, rather than the anchor, leads to more effective training by avoiding overly aggressive hard negative pairs that could potentially destabilize representation learning.

\begin{table}[h]
    \begin{minipage}{0.55\textwidth}
        \centering
        \footnotesize
        \setlength{\tabcolsep}{2pt}  
        \caption{Ablation study}
            \begin{tabular}{@{}lccc@{}}  
            \toprule
            Method & MAE ${\downarrow}$ & MSE ${\downarrow}$ & GM ${\downarrow}$\\
            \midrule
            SupCon  & 6.79 & 76.5 & 4.89 \\
            SupCon+DM & 6.45 & 70.2 & 4.55 \\
            SupCon+Mix-neg & 5.92 & 65.8 & 4.32 \\
            SupCon+Mix-neg+DM & 4.85 & 52.4 & 3.85 \\
            \textbf{SupReMix} & \textbf{4.08} & \textbf{43.7} & \textbf{3.22} \\
            \bottomrule
            \end{tabular}
        \label{table9}
    \end{minipage}%
    \begin{minipage}{0.45\textwidth}
        \centering
        \footnotesize
        \setlength{\tabcolsep}{3pt}  
        \caption{Choice of window size $\gamma$}
            \begin{tabular}{@{}cc|cc@{}}  
            \toprule
            RSNA & MAE ${\downarrow}$ & UKB & MAE ${\downarrow}$\\
            \midrule
             $\gamma=1$ & 4.86 & $\gamma=0.2$ & 3.15 \\
             $\gamma=3$ & 4.45 & $\gamma=0.5$ & 3.08 \\
             $\gamma=5$ & \textbf{4.08} & $\gamma=1.0$ & \textbf{2.97} \\
             $\gamma=10$ & 4.95 & $\gamma=2.0$ & 3.22 \\
             $\gamma=\infty$ & 5.80 & $\gamma=\infty$ & 3.45 \\
            \bottomrule
            \end{tabular}
        \label{table10}
    \end{minipage}
\end{table}

\begin{table}[htbp]
    \centering 
    \caption{Choice of beta distribution for sampling mixing ratios in Mix-neg generation. Results shown as MAE / MSE across six medical imaging regression datasets.}
    \newcolumntype{A}{ >{\centering\arraybackslash} m{4.0cm} }
    \newcolumntype{B}{ >{\centering\arraybackslash} m{0.8cm} }
    \newcolumntype{C}{ >{\centering\arraybackslash} m{0.8cm} }
    \newcolumntype{D}{ >{\centering\arraybackslash} m{0.8cm} }
    \newcolumntype{E}{ >{\centering\arraybackslash} m{0.8cm} }
    \newcolumntype{F}{ >{\centering\arraybackslash} m{0.8cm} }
    \newcolumntype{G}{ >{\centering\arraybackslash} m{0.8cm} }
    \begin{tabular}{ABCDEFG}
    \toprule 
    $(\alpha,\beta)$ & UKB & HCP & RSNA & RHPE & Echo & A4 \\ 
    \hline
    \includegraphics[width=1\linewidth, height=0.25\linewidth]{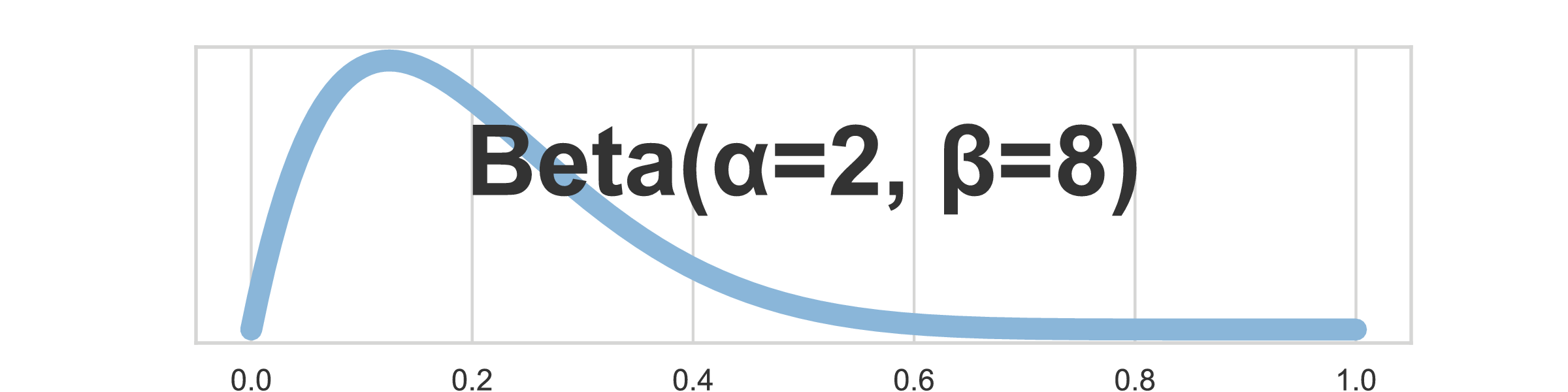} & \textbf{\shortstack{3.8\\14.5}} & \textbf{\shortstack{5.8\\65.2}} & \textbf{\shortstack{4.2\\45.8}} & \textbf{\shortstack{5.5\\64.2}} & \textbf{\shortstack{4.2\\20.5}} & \textbf{\shortstack{0.52\\0.004}} \\ 
    \hline
    \includegraphics[width=1\linewidth, height=0.25\linewidth]{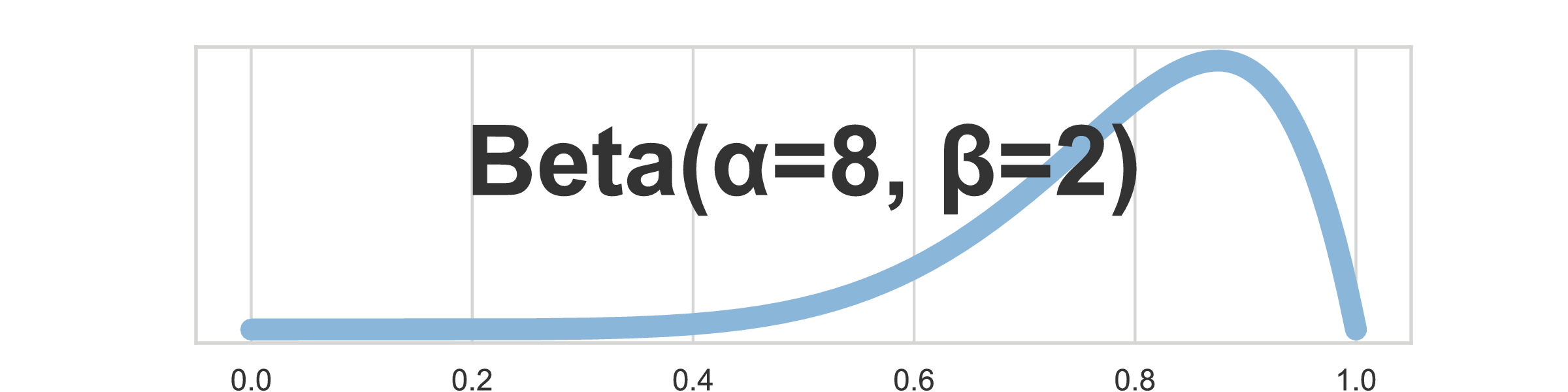} & \shortstack{4.1\\15.2} & \shortstack{6.2\\68.5} & \shortstack{4.5\\48.2} & \shortstack{5.8\\68.5} & \shortstack{4.5\\22.8} & \shortstack{0.56\\0.006} \\
    \hline
    \includegraphics[width=1\linewidth, height=0.25\linewidth]{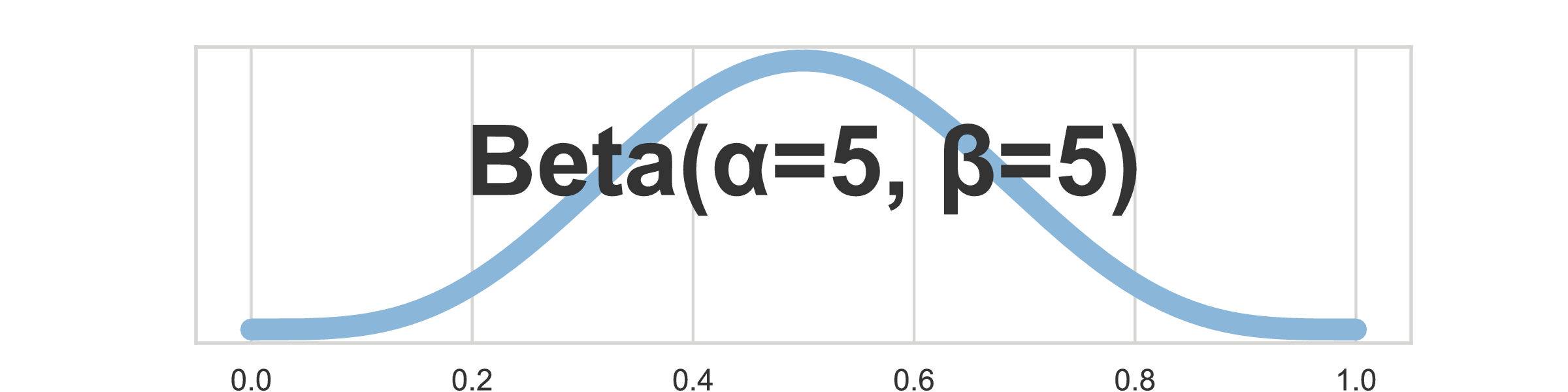} & \shortstack{3.95\\14.8} & \shortstack{6.0\\66.8} & \shortstack{4.35\\47.0} & \shortstack{5.65\\66.4} & \shortstack{4.35\\21.6} & \shortstack{0.54\\0.005} \\
    \bottomrule
    \end{tabular}
    \label{table11}
\end{table}

\section{Discussion and conclusion}

In this work, we targeted medical image regression tasks, which often receive less attention compared to classification problems. When examining learned representations in regression tasks like bone age estimation, we identified two key challenges: (1) ordinality-awareness in the representation space (Section 5.1.1) and (2) hardness of contrastive pairs (Section 5.1.2). To address these limitations, we developed SupReMix, a supervised contrastive learning framework specifically designed for regression tasks that incorporates mixup techniques at the embedding level.

SupReMix addresses the regression representation challenge through two components. First, it implements anchor-inclusive mixtures as hard negatives and anchor-exclusive mixtures as hard positives, enhancing the continuity and local linearity of the learned representations. Second, it incorporates label distance information through distance magnifying weights, explicitly capturing the ordinal relationships inherent in regression tasks. Moreover, the theoretical analysis shows that SupReMix promotes globally ordered and locally linear representations, which align with the continuous nature of regression problems.

Experiments across six diverse medical imaging datasets (Section 4.1) confirm SupReMix's effectiveness, consistently outperforming both classification-based methods (SimCLR, SupCon) and regression-specific approaches (AdaCon, RNC) as shown in Section 5.2. On the RSNA bone age dataset, SupReMix achieved an MAE of 4.08 months compared to SupCon's 6.79 months—a 39.9\% improvement. Our analysis of representation continuity through Normalized Lipschitz Factor Distribution (Section 5.3) provides quantitative evidence that SupReMix learns smoother representations that better preserve ordinal relationships, with t-SNE visualizations confirming more coherent, ordered embeddings compared to classification-based methods.

The performance improvements demonstrated by SupReMix directly translate to enhanced clinical diagnostics across medical domains. For instance, the reduced error in bone age assessment enables more precise growth disorder diagnosis, while improved brain age prediction and cardiac function estimation support earlier neurological intervention and more accurate heart failure management, respectively. More importantly, SupReMix offers strong generalizability for cross-site deployment, a critical advantage in healthcare where model adaptation between different hospitals and imaging equipment often degrades performance. Our transfer learning experiments (Section 5.6) suggest that models pretrained with SupReMix maintain reliability when adapting across institutions with different imaging protocols and patient demographics. Additionally, SupReMix's effectiveness in low-resource settings (Section 5.5) addresses a fundamental challenge in clinical implementation, where smaller hospitals may have access to limited labeled cases but could still benefit from AI-assisted diagnostics for specialized tasks like brain age estimation or cardiac function assessment.

While SupReMix demonstrates significant advantages in the scenarios above, it is important to acknowledge areas where further develepment could be beneficial. The application of SupReMix faces limitations in its adaptability to higher-dimensional regression labels. This challenge stems from a fundamental issue: while ordinality is crucial in regression tasks, it is undefined for vectors in dimensions greater than one, as the topology on $\mathbb{R}^n$ does not form an order topology. In Mix-pos, obtaining a weights vector (whose dimension matches that of the label) by solving a linear system is theoretically feasible. However, this solution is not always assured due to potential linear independence between selected samples and the anchor. Moreover, the approach's scalability is hindered as the dimensionality of the regression label increases, posing significant practical challenges. Future research could focus on exploring methods to preserve ordinality in regression representations when dealing with higher-dimensional labels.

To conclude, in this paper, we propose \emph{Supervised Contrastive Learning for Medical Imaging Regression with Mixup (SupReMix)}, a novel framework that generates hard negatives and hard positives for supervised contrastive regression. Supported by theoretical analysis, SupReMix leads to continuous ordered representations for regression. Extensive experiments on six different medical imaging datasets have shown that SupReMix consistently improves over baselines, including vanilla deep regression, previous contrastive learning frameworks, and task-specific methods, across datasets, tasks, and input modalities. Beyond its core performance, SupReMix shows robust generalization when handling challenging scenarios like missing targets and few-shot cases, while also proving valuable as a pre-training strategy for existing architectures.

\bibliographystyle{elsarticle-num} 
\bibliography{main}
\newpage

\appendix
\newgeometry{left=0.7in,right=0.7in,top=1in,bottom=1in}
\section{Proof}

In this appendix, $z_{m,i}$ is used to denote normalized embeddings, whereas $z^*_{m,i}$ is used to denote unnormalized embeddings, in other words, we have $z_{m,i} = \frac{z^*_{m,i}}{||z_{m,i}||}$.
\label{proof}

\subsection{Proof of Theorem~\ref{th1}}
\begin{proof}
For $m\neq m'$, we have
$$
\frac{\partial \mathcal{L}}{\partial s^{m,m'}_{i,j}} = \frac{b_{m,i}\cdot w_{m,m'}\exp(\langle z_{m,i}, \ z_{m',j}\rangle/\tau)}{C^w_{m,i}\cdot \tau} > 0
$$

where
$$
C^w_{m,i} : = \sum_{n\in \overline{M}}\sideset{}{'}\sum_{l=1}^{k_{(m,i),n}}w_{m,n}\exp(\langle z_{m,i}, \ z_{n,l} \rangle /\tau), \ b_{m,i}=\frac{k_{(m,i),m}-1}{k_m}.
$$

Consider $w_{m,m'}=\frac{1+t|m-m'|}{m_{\text{max}}-m_{\text{min}}}$. For comparison, we have
\begin{equation}
\begin{split}
\frac{\frac{\partial \mathcal{L}}{\partial s^{m,m'}_{i,j}}}{\frac{\partial \mathcal{L}}{\partial s^{m,m''}_{i,l}}} & = \frac{b_{m,i} \cdot w_{m,m'}\exp(\langle z_{m,i},\ z_{m',j}\rangle/\tau)}{C^w_{m,i}\cdot \tau}\cdot\frac{C^w_{m,i}\cdot \tau}{b_{m,i} \cdot w_{m,m''}\exp(\langle z_{m,i},\ z_{m'',l}\rangle/\tau)} \\
& = \frac{w_{m,m'}}{w_{m,m''}}\cdot\exp(z_{m,i}\cdot(z_{m',j}-z_{m'',l})/\tau)\\
& = \frac{1+t|m-m'|}{1+t|m-m''|}\cdot \exp(z_{m,i}\cdot(z_{m',j}-z_{m'',l})/\tau),
\end{split}
\end{equation}
Then we have
$$
\left.\frac{\frac{\partial \mathcal{L}}{\partial s^{m,m'}_{i,j}}}{\frac{\partial \mathcal{L}}{\partial s^{m,m''}_{i,l}}}\right|_{t=1}>\left.\frac{\frac{\partial \mathcal{L}}{\partial s^{m,m'}_{i,j}}}{\frac{\partial \mathcal{L}}{\partial s^{m,m''}_{i,l}}}\right|_{t=0}
$$

\end{proof}

\subsection{Proof of Lemma~\ref{le1}}

\begin{proof}
Recall we have 
$$
            \mathcal{L}_{\text{SupReMix}} = \sum_{m\in M}\frac{1}{k_m}\sum_{i=1}^{k_m}\sum_{\substack{j=1\\j\neq i}}^{k_{(m,i),m}}\log\frac{\sum_{\overline{m}\in \overline{M}}\sum_{l=1}^{\prime k_{(m,i),\overline{m}}}
            w_{m,\overline{m}}\exp(\langle z_{m,i},\ z_{\overline{m},l} \rangle /\tau)}{\langle \exp(z_{m,i}, z_{m,j} \rangle/\tau)}
$$

Since the logarithmic function is monotone, and both weight and exponential function are positive, we keep only the positive pairs in the numerators and have

\begin{equation}\label{ineq1}
\begin{split}
    \mathcal{L}_{\text{SupReMix}}  & \geq \sum_{m\in M}\frac{1}{k_m}\sum_{i=1}^{k_m}\sum_{\substack{j=1\\j\neq i}}^{k_{(m,i),m}}\log\frac{\sum_{\substack{l=1\\l\neq i}}^{ k_{(m,i),m}}
            \frac{1}{m_{\text{max}}-m_{\text{min}}}\exp(\langle z_{m,i},\ z_{m,l}\rangle /\tau)}{\exp(\langle z_{m,i},\ z_{m,j}\rangle /\tau)}\\
            & = -\sum_{m\in M}\frac{1}{k_m}\sum_{i=1}^{k_m}\sum_{\substack{j=1\\j\neq i}}^{k_{(m,i),m}}\log\frac{\exp(\langle z_{m,i},\ z_{m,j}\rangle /\tau)}{\sum_{\substack{l=1\\l\neq i}}^{ k_{(m,i),m}}
            \frac{1}{m_{\text{max}}-m_{\text{min}}}\exp(\langle z_{m,i},\ z_{m,l}\rangle /\tau)}\\
            &= -\sum_{m\in M}\frac{1}{k_m}\sum_{i=1}^{k_m}\sum_{\substack{j=1\\j\neq i}}^{k_{(m,i),m}}\log\frac{\frac{1}{m_{\text{max}}-m_{\text{min}}}\exp(\langle z_{m,i},\ z_{m,j}\rangle/\tau)}{\sum_{\substack{l=1\\l\neq i}}^{ k_{(m,i),m}}
            \frac{1}{m_{\text{max}}-m_{\text{min}}}\exp(\langle z_{m,i},\ z_{m,l}\rangle /\tau)}+C,
\end{split}
\end{equation}
where $C = -\log(m_{\text{max}}-m_{\text{min}})\cdot\sum_{m\in M}\frac{1}{k_m}\sum_{i=1}^{k_m}(k_{(m,i),m}-1)$ is a constant. Now since $-\log$ is convex, by Jensen's inequality, we have

\begin{equation}\label{ineq2}
\begin{split}
\mathcal{L}_{\text{SupReMix}} & = -\sum_{m\in M}\frac{1}{k_m}\sum_{i=1}^{k_m}(k_{(m,i),m}-1)\sum_{\substack{j=1\\j\neq i}}^{k_{(m,i,m}}\frac{1}{(k_{(m,i),m}-1)}\log\frac{\frac{1}{m_{\text{max}}-m_{\text{min}}}\exp(\langle z_{m,i},\ z_{m,j}\rangle /\tau)}{\sum_{\substack{l=1\\l\neq i}}^{ k_{(m,i),m}}
            \frac{1}{m_{\text{max}}-m_{\text{min}}}\exp(\langle z_{m,i}, z_{m,l}\rangle /\tau)}+C\\
& \geq -\sum_{m\in M}\frac{1}{k_m}\sum_{i=1}^{k_m}(k_{(m,i),m}-1)\log\left(\frac{\sum_{\substack{l=1\\l\neq i}}^{ k_{(m,i),m}}
            \frac{1}{m_{\text{max}}-m_{\text{min}}}\exp(\langle z_{m,i},\ z_{m,l}\rangle /\tau)}{( k_{(m,i),m}-1)\sum_{\substack{l=1\\l\neq i}}^{ k_{(m,i),m}}
            \frac{1}{m_{\text{max}}-m_{\text{min}}}\exp(\langle z_{m,i}, z_{m,l}\rangle /\tau)}\right)+C\\
            & = \sum_{m\in M}\sum_{i=1}^{k_m}\frac{(k_{(m,i),m}-1)\log(k_{(m,i),m}-1)}{k_m}+C\\
            & = \sum_{m\in M}\frac{1}{k_m}\sum_{i=1}^{k_m}(k_{(m,i),m}-1)\log\frac{k_{(m,i),m}-1}{m_{\text{max}}-m_{\text{min}}}=:\mathcal{L}^*
\end{split}
\end{equation}

\end{proof}

\subsection{Proof of Theorem~\ref{th2}}

\begin{proof}
    Let $\epsilon, \delta > 0$ be given. We aim to demonstrate the existence of a positive real number $\tau_0$ such that for all $\tau > \tau_0$, embeddings with $\mathcal{L} < \mathcal{L}^* + \epsilon$ can be obtained with a probability exceeding $1 - \delta$.

    Firstly, we assign identical embedding $z_m$ to all samples labeled with $m$ belonging to the set $M$. Subsequently, we set these embeddings, denoted as $z_m$, to lie on a common plane. Without loss of generality, we can assume this plane to be spanned by the first two coordinate axes, thereby effectively reducing the embedding space to two dimensions. 

    Consider $z_m^*$ as the anchor embedding before normalization. We further take the $z_m^*$ to lie on a line such that 
    their position on the line is proportional to their label, i.e.,
$$
\frac{m-m'}{m-m''} = \frac{||z^*_m-z^*_{m'}||}{||z^*_m-z^*_{m''}||}, \quad \forall m \neq m' \neq m''\in M
$$
    then the Mix-pos associated with label $m$ will share the same embedding, $z_m$. When it comes to Mix-neg, note that for any sample $(m,i)$, only a finite number of Mix-neg exist. We define an event $E_{m,i}$ to occur if at least one Mix-neg of $(m,i)$ forms an angle smaller than $\theta_{m,i}$ with the anchor $(m,i)$. We choose $\theta_{m,i}$ such that the event $E_{m,i}$ has a probability less than $\frac{\delta}{N}$. This ensures that the cumulative probability of all $E_{m,i}$ events being false surpasses $1 - \delta$. Define $\theta_0$ to be the minimum of $\theta_{m,i}$ over all pairs $(m,i)$ in set $I$ and the angles between all pairs of $z_{m},z_{m'}$, represented mathematically as:
    \begin{equation}
        \theta_0:=\min\{\min_{(m,i)\in I}\{\theta_{m,i}\},\min_{m,m'\in M}\arccos(z_{m}\cdot z_{m'})\}
    \end{equation}
    Upon normalization, any pair of distinct embeddings will maintain an angular separation of at least $\theta_0$. Then we have

    \begin{equation}\label{ineq3}
\begin{split}
            \mathcal{L}_{\text{SupReMix}} = &   \sum_{m\in M}\frac{1}{k_m}\sum_{i=1}^{k_m}\sum_{\substack{j=1\\j\neq i}}^{k_{(m,i),m}}\log\left(\frac{\sum_{\overline{m}\neq m}\sum_{l=1}^{\prime k_{(m,i),\overline{m}}}
            w_{m,\overline{m}}\exp(\langle z_{m,i},\ z_{\overline{m},l}\rangle/\tau)}{\exp(1/\tau)}\right.\\
            & +  \left.\frac{\sum_{\substack{l=1\\l\neq i}}^{ k_{(m,i),m}}\frac{1}{m_{\text{max}}-m_{\text{min}}}\exp(1/\tau)}{\exp(1/\tau)}\right)\\
            < &  \sum_{m\in M}\frac{1}{k_m}\sum_{i=1}^{k_m}\sum_{\substack{j=1\\j\neq i}}^{k_{(m,i),m}}\log\left(\frac{\sum_{\overline{m}\neq m}\sum_{l=1}^{ k_{(m,i),\overline{m}}}\exp(\cos(\theta_0)/\tau)+\frac{(k_{(m,i),m}-1)\exp(1/\tau)}{m_{\text{max}}-m_{\text{min}}}}{\exp(1/\tau)}\right)\\
            = & \mathcal{L}^*+ \sum_{m\in M}\frac{1}{k_m}\sum_{i=1}^{k_m}\sum_{\substack{j=1\\j\neq i}}^{k_{(m,i),m}}\log\left(1+\frac{(m_{\text{max}}-m_{\text{min}})\sum_{\overline{m}\neq m}\sum_{l=1}^{ k_{(m,i),\overline{m}}}\exp((\cos(\theta_0)-1)/\tau)}{(k_{(m,i),m}-1)}\right)\\
         \end{split}
         \end{equation}

Now since $\tau \rightarrow \infty$, we have the log functions in the second term converge to $0$, then for any $\epsilon>0$ there exists $\tau_0>0$, for all $\tau>\tau_0$, we have

\begin{equation}
\sum_{m\in M}\frac{1}{k_m}\sum_{i=1}^{k_m}\sum_{\substack{j=1\\j\neq i}}^{k_{(m,i),m}}\log\left(1+\frac{(m_{\text{max}}-m_{\text{min}})\sum_{\overline{m}\neq m}\sum_{l=1}^{ k_{(m,i),\overline{m}}}\exp((\cos(\theta_0)-1)/\tau)}{(k_{(m,i),m}-1)}\right)<\epsilon,
\end{equation}
therefore, we have
$$
\mathcal{L}_{\text{SupReMix}}<\mathcal{L}^*+\epsilon.
$$

 Based on the above and the proof of lemma 1, we know that the loss function is closed to its infimum if, and only if, the following are true: 
\begin{enumerate}
     \item all the real samples with the same label $m$ are embedded close to some vector $z_m$;
     \item all Mix-pos of an anchor $(m,i)$ have embeddings close to $z_m$;
     \item all the negatives (real and Mix-neg) of an anchor $(m,i)$ have $z_{m,i}\cdot z_{m',j}$ that are not equal to $1$.
\end{enumerate}
Conditions 1 and 2 are both necessary and sufficient for the inequality in \eqref{ineq2} to closely approach equality. On the other hand, condition 3, specified as \(\cos{\theta_0} \neq 1\), is the necessary and sufficient condition for inequality in \eqref{ineq3} — an angular version of \eqref{ineq2} — to similarly approach equality.
 Moreover, condition 2 is true when for any anchor $(m,i)$, all the samples with label $m'$ such that $m-\epsilon<m'<m+\epsilon$ are closed to a line, and their positions on the line are proportional to their labels. In other words, they are locally ordered and linear. Finally, condition 3 holds when the negative pairs are apart from each other, combined with condition 2 shows $z_m$ are globally ordered. Therefore, we can see the loss function approaches its infimum when the embeddings are globally ordered and locally linear.
\end{proof}

\subsection{Risk Bound Analysis}

In this section, we analyze the generalization bound from SupReMix following \citep{zha2023rank}. 

Given a regression task with training set $S = \{ (x_i, y_i) \}_{i=1}^{N}$, let $H_1$ be the hypothesis set containing all possible mapping functions from $x_i$ to $y_i$, $f$ be an encoder mapping $x_i$ to $z_i$, $g(z_i)=y_i$. An embedding set is called "$\epsilon-\text{ordered}$" if, for any label value $m$, there exists embedding $z_m$ such that any $x_i$ with $y_i=m$, we have $|z_i-z_m|<\epsilon$, and for any $x_i$, $x_j$ with $y_i\neq y_j$ we have $|z^T_i\cdot z_j|<1-\epsilon$. With SupReMix, $f$ is guaranteed to map $x$ to $\epsilon-\text{ordered}$ set. We denote the class of all possible $h$ with $f$ that could lead to $\epsilon-\text{ordered}$ set by $H_2$. Both $H_1$ and $H_2$ contain the optimal hypothesis $h^{*}$ such that for any $x,y$, $h^{*}(x)=y$.

Denote $u_i$ as the upper bound of the loss from $(x_i,y_i)$, the set of loss from the training set as $A_k$ for each hypothesis in a hypothesis set $H_k$, with Rademacher Complexity \cite{bartlett2002rademacher} $R(A_k)$. The gap between training error and test error is upper bounded by $2R(A_i)+4u_i\sqrt{2\ln(4/\delta)/N}$, with probability $1-\delta$. Since $H_2 \subset H_1$, we have $A_2 \subset A_1$ and $u_2 \leq u_1$. By the monotonicity of Rademacher Complexity we have $R(A_2)\leq R(A_1)$.

\newpage

\section{Dataset}

\subsection{Dataset Details}
\label{B1}
\vspace{3cm}
\begin{table}[h!]
    \vspace{-3cm}
    \centering
    \caption{Overview of the six medical imaging datasets used in our experiments}
    \resizebox{1.0\textwidth}{!}{%
    \begin{tabular}{|c|c|c|c|c|c|c|}
    \hline Dataset & Target type & Target range & \# Training set & \# Val. set & \# Test set & Modality \\
    \hline \hline
    UK Biobank & Brain age & 42-82 years& 19,509 & 2,434 & 2,431 & MRI T1\\
    \hline
    HCP-Lifespan & Brain age & 36-100 years& 456 & 100 & 100 & Rs-fMRI \\
    \hline
    RSNA & Bone age & 1-228 months & 11,611 & 1,000 & 200 & X-ray \\
    \hline
    RHPE & Bone age & 10-242 months & 5,496 & 716 & 80 & X-ray \\
    \hline
    Echo-Net & Ejection fraction & 6.91-96.96 \%& 7,465 & 1,288 & 1,277 & Echocardiogram \\
    \hline
    A4 & Amyloid SUVR & 0.45-2.58 g/m& 3,486 & 500 & 500 & PET \\
    \hline
    \end{tabular}
    }
\end{table}

\subsubsection{UK Biobank}
UK Biobank is a large-scale prospective epidemiological study containing comprehensive health and medical data from over 500,000 participants across United Kingdom. From this rich database, we utilize the T1-weighted structural MRI scans collected from 19,509 participants aged 42-82 years to evaluate brain age prediction. Our preprocessing pipeline begins with careful quality control to exclude subjects with current/past stroke, cancer, long standing illness and poor health rating. The original non-skull-stripped T1-weighted images, initially at 1×1×1 mm³ resolution, are resampled to 2×2×2 mm³ to optimize computational efficiency while preserving anatomical detail. We perform brain extraction using FSL BET with a fractional intensity threshold of 0.5, followed by bias field correction using the N4 algorithm. The images are then linearly registered to MNI152 space using FSL FLIRT with 12 degrees of freedom. To standardize the input size, we apply center cropping to achieve dimensions of 100 x 100 x 100, followed by intensity normalization to zero mean and unit variance. These preprocessing steps follow the established UK Biobank pipeline described in Alfaro-Almagro et al. (2018).

\subsubsection{HCP-Lifespan}
The HCP-Lifespan dataset provides resting state fMRI data designed to study the aging process from middle age to older adulthood. Our preprocessing approach implements a comprehensive pipeline for the fMRI data. We begin with motion correction using FSL MCFLIRT, followed by slice timing correction through FSL slicetimer. Global signal regression is performed following the methodology of Li et al. (2019). The data undergoes temporal filtering with a 0.01-0.1 Hz bandpass filter and spatial smoothing using a 6mm FWHM Gaussian kernel. We then apply parcellation using the Schaefer-400 atlas, resulting in 400 regions of interest (ROIs). Time series are extracted from each ROI, producing final data dimensions of [400 (parcells) × 478 (time frames)]. Quality control measures include the exclusion of subjects with excessive head motion (mean FD $>$ 0.3mm) or incomplete scans.

\subsubsection{RSNA \& RHPE}
For bone age assessment, we utilize two distinct X-ray datasets: the RSNA Bone Age Challenge dataset and the RHPE dataset. For the RSNA dataset, preprocessing begins with converting DICOM images to PNG format, followed by contrast enhancement using histogram equalization. Images are then resized to 520×400 pixels using bilinear interpolation. During training, we employ data augmentation techniques including random rotation within ±10 degrees, random horizontal flipping, brightness and contrast adjustments, and random cropping with padding. Finally, intensity values are normalized to the range [0,1].

The RHPE dataset requires additional preprocessing steps due to its dual hand radiographs. We first isolate the left hand by splitting the original images, then apply ground truth bounding boxes to crop the hand region. Contrast enhancement is performed using CLAHE (Contrast Limited Adaptive Histogram Equalization). To maintain consistency with the RSNA dataset, images are resized to 520×400 pixels. We apply similar data augmentation strategies as with RSNA, with additional background standardization to reduce variability across images.

\subsubsection{Echo-Net}
The Echo-Net dataset preprocessing begins with converting DICOM videos to image sequences and carefully removing ECG traces and text overlays through mask-based filtering. We extract frames at 30fps and rescale all frames to 112×112 pixels, followed by intensity normalization to zero mean and unit variance. The temporal preprocessing focuses on selecting frames that cover a complete cardiac cycle, with temporal resampling applied to standardize sequence length. During training, we implement frame jitter augmentation to enhance model robustness. Our quality control process is thorough, removing studies with poor image quality, verifying correct cardiac view (apical-4-chamber), and confirming complete cardiac cycle coverage. These steps ensure consistent, high-quality data for training and evaluation.

\subsubsection{A4}
The A4 study PET data preprocessing involves multiple stages optimized for amyloid burden quantification. The dynamic PET acquisition consists of four 5-minute frames collected between 50-70 minutes post-injection using [18F]-Florbetapir (FBP) tracer. Motion correction is performed through frame-to-frame realignment and mean image calculation for each subject. Spatial normalization includes registration to the MNI152 template and application of standard space transformations.

For SUVR calculation, we use the whole cerebellum as the reference region and extract measurements from six key cortical regions: medial orbital frontal, temporal, parietal, anterior cingulate, posterior cingulate, and precuneus. The resulting values are normalized to a range of 0.45 to 2.58. Our quality control process involves excluding scans with excessive motion, verifying complete brain coverage, and checking for artifacts and signal abnormalities. All SUVR calculations and regional measurements follow standard processing pipelines for amyloid PET quantification. We employ a predefined amyloid positivity (A$\beta$+) cutoff of $>$ 1.15 for classifying amyloid burden status.

\subsection{Ethic Statements}
\label{B2}
All datasets used in our experiments are publicly available and have been properly de-identified to protect patient privacy. Access to the UK Biobank data is available through a formal application process via their website (https://www.ukbiobank.ac.uk/), with all subject information thoroughly de-identified. The HCP-Lifespan dataset can be accessed through the Human Connectome Project website (https://www.humanconnectome.org/lifespan-studies), providing resting-state fMRI data that has been preprocessed to remove any identifying information. The RSNA Bone Age and RHPE datasets are publicly available through their respective challenges, containing only anonymized hand radiographs. The Echo-Net dataset is accessible through the Stanford Digital Repository, featuring anonymized echocardiogram videos. The A4 study data can be obtained through the Laboratory of Neuro Imaging (LONI) platform after completing appropriate data use agreements, with all PET scans being fully de-identified before distribution. These measures ensure compliance with ethical guidelines while facilitating reproducible research.

\section{Experimental Settings}

\subsection{Implementation Details}
\label{C1}

\begin{table}[htbp]
    \centering
    \caption{Detailed hyperparameters and implementation settings across datasets}
    \resizebox{\textwidth}{!}{%
    \begin{tabular}{@{}lcccccc@{}}
    \toprule
    & UK Biobank & HCP-Lifespan & RSNA & RHPE & Echo-Net & A4 \\
    \midrule
    Base encoder & 3D ResNet-18 & 1D ResNet-18 & ResNet-50 & ResNet-50 & r2plus1d-18 & 3D ResNet-18 \\
    Feature dim & 128 & 128 & 128 & 128 & 128 & 128 \\
    Batch size & 64 & 64 & 128 & 128 & 64 & 32 \\
    Learning rate & 1e-3 & 1e-3 & 1e-3 & 1e-3 & 1e-3 & 1e-3 \\
    Weight decay & 1e-4 & 1e-4 & 1e-4 & 1e-4 & 1e-4 & 1e-4 \\
    Optimizer & Adam & Adam & Adam & Adam & Adam & Adam \\
    Temperature ($\tau$) & 0.5 & 0.5 & 1.0 & 1.0 & 1.0 & 1.0 \\
    Beta dist.\ params & (2,8) & (2,8) & (2,8) & (2,8) & (2,8) & (2,8) \\
    Window size ($\gamma$) & 3 & 5 & 5 & 5 & 7 & 7 \\
    Input size & 100$\times$100$\times$100 & 400$\times$T & 520$\times$400 & 520$\times$400 & 112$\times$112$\times$T & 100$\times$100$\times$100 \\
    Training epochs & 200 & 200 & 200 & 200 & 200 & 200 \\
    \bottomrule
    \end{tabular}%
    }
\end{table}

\subsubsection{Network Architecture Details}

\paragraph{3D ResNet-18 (UK Biobank, A4):}
The 3D ResNet-18 architecture was modified specifically for volumetric medical images. The initial convolution layer uses a kernel size of 7×7×7 with stride 2 and 64 output channels. This is followed by a max pooling layer with kernel 3×3×3 and stride 2. The network contains four residual blocks with channel dimensions progressing through [64,128,256,512]. Each convolutional layer is followed by BatchNorm3d and ReLU activation. The final features are processed by average pooling before entering the projection head. The projection head architecture consists of three layers: a linear transformation from 512 to 2048 dimensions, ReLU activation, and a final linear layer reducing to 128 dimensions.

\paragraph{1D ResNet-18 (HCP-Lifespan):}
For time-series fMRI data processing, we implemented a customized 1D ResNet-18 that handles 400 ROIs. The architecture begins with an initial 1D convolution using kernel size 7 and stride 2, transforming the input from 400 channels to 64 feature channels. Four temporal residual blocks process the time dimension while maintaining the ROI structure. The network employs adaptive average pooling to accommodate variable sequence lengths. The projection head maintains consistency with other architectures, using identical dimensions of 2048-128.

\paragraph{ResNet-50 (RSNA, RHPE):}
Our implementation of ResNet-50 was adapted for grayscale medical images with several key modifications. The input convolution layer was adjusted to accept single-channel images while maintaining the standard bottleneck blocks in a [3,4,6,3] layer configuration. The channel dimensions progress through 64→256→512→1024→2048. We implemented additional padding in the first layer to properly handle the high-resolution radiographs, and adjusted stride patterns throughout the network to maintain appropriate feature map resolution. The network concludes with global average pooling before the projection head.

\paragraph{r2plus1d-18 (Echo-Net):}
The r2plus1d-18 architecture was specifically configured for echocardiogram video analysis. The network decomposes 3D convolutions into separate spatial (2D) and temporal (1D) components. The initial block combines a spatial 7×7 convolution with a temporal 3×1 convolution, followed by 3D max pooling. Four main blocks follow with channel dimensions of [64,128,256,512]. The network incorporates spatio-temporal pooling to handle variable-length sequences, with temporal stride parameters optimized for 30fps input processing.

\subsubsection{Training Protocol Details}

\paragraph{Optimization:}
Our implementation uses the Adam optimizer with beta1 set to 0.9, beta2 to 0.999, and epsilon to 1e-8. The learning rate schedule incorporates a warmup period over the first 10 epochs, followed by cosine decay according to the formula $\eta_t = \eta_{min} + \frac{1}{2}(\eta_{max} - \eta_{min})(1 + \cos(\frac{t\pi}{T}))$. We apply gradient clipping with a maximum norm of 1.0 and weight decay of 1e-4 to all parameters except bias terms and BatchNorm layers. For multi-GPU training, we utilize synchronized BatchNorm to maintain consistent statistics across devices.

\paragraph{Data Loading:}
Data loading is optimized for each modality. Volumetric data from UK Biobank and A4 is loaded using SimpleITK with per-volume normalization. HCP time series data is processed in chunks with per-ROI standardization. RSNA and RHPE images are loaded via PIL with intensity scaling to the range [0,1]. Echo-Net data loading implements sequential frame extraction for video processing. Our data pipeline utilizes prefetching with 4 worker threads and memory pinning for optimal GPU transfer speeds.

\paragraph{Hardware and Software:}
All experiments were conducted on NVIDIA A100 GPUs with 40GB memory, using CUDA 11.7 and PyTorch 1.13.1 on Python 3.8. The training pipeline implements mixed precision training through torch.cuda.amp and distributed training via DistributedDataParallel with gradient synchronization occurring every step. 

\paragraph{Model Checkpointing:}
Model checkpointing saves the best performing models based on validation MAE, with checkpoints recorded every 10 epochs. Each checkpoint contains the complete model state dictionary, optimizer state, current epoch number, and best achieved metrics. We implement exponential moving average (EMA) model averaging with a momentum value of 0.999. The automatic mixed precision state is preserved in checkpoints to ensure training continuity.

For specific implementation details not covered in this documentation, please refer to our released codebase.

\section{Additional Results and Analysis}
\label{D}

\subsection{Generalization on missing targets}
\label{Generalization}
Regression datasets often suffer from "missing targets", where samples with certain target values are absent in the training set. To explore performance in this scenario, we curate subsets of the UK Biobank dataset by removing samples within specific age ranges (45-50, 60-65, and 75-80 years) while maintaining the original validation and test sets.

Table \ref{table15} illustrates that SupReMix significantly outperforms the Vanilla approach overall, improving MAE by 6.2\%. More remarkably, it boosts performance under missing targets setting by 24.5\%. This improvement stems from SupReMix's ability to learn more continuous representations, leveraging mixtures as effective landmarks for learning representations of missing targets, thus enhancing prediction accuracy with unseen data.

\begin{table}[h]
    \centering
        \caption{Evaluation on UK Biobank with missing targets (MT)}
        \begin{tabular}{lccccc}
        \toprule & \multicolumn{2}{c}{ MAE $\downarrow$} & & \multicolumn{2}{c}{ GM $\downarrow$} \\
        \cline { 2 - 3 } \cline { 5 - 6 } & Overall & MT & & Overall & MT \\
        \hline
        Vanilla & 3.12 & 4.45 & & 14.89 & 19.23 \\
        SupCon \citep{khosla2020supervised} & 3.08 & 4.12 & & 14.52 & 17.56 \\
        RNC \citep{zha2023rank} & 3.15 & 4.28 & & 14.92 & 18.12 \\
        SupReMix & \textbf{2.93} & \textbf{3.36} & & \textbf{13.85} & \textbf{15.23} \\
        \hline 
        GAINS (\textbf{Ours} VS. Vanilla(\%)) & \textbf{\textcolor{darkergreen}{+6.2}} & \textbf{\textcolor{darkergreen}{+24.5}} & & \textbf{\textcolor{darkergreen}{+7.0}} & \textbf{\textcolor{darkergreen}{+20.8}} \\
        \bottomrule
        \end{tabular}
    \label{table15}
\end{table}

\subsection{Varying batch size}
\label{varybatch}
In contrastive learning, pioneered by studies such as \citep{chen2020simple} and supervised contrastive learning \citep{khosla2020supervised}, large batch sizes have been a common strategy to maintain extensive negative sample pools. We investigate batch size impact across medical imaging modalities using the UK Biobank and Echo-Net datasets. Results with batch sizes from 32 to 512 show that, contrary to classification tasks, increasing batch size does not consistently improve performance for regression tasks.

\begin{table}[h]
    \centering
        \caption{Effect of varying batch size on UK Biobank and Echo-Net datasets}
        \begin{tabular}{cccccc}
        \toprule & \multicolumn{2}{c}{ MAE $\downarrow$} & & \multicolumn{2}{c}{ GM $\downarrow$}  \\  
        \cline { 2 - 3 } \cline { 5 - 6 } \vspace{-0.3cm} \\ 
        Batch Size & UK Biobank & Echo-Net & & UK Biobank & Echo-Net \\
        \hline
        32  & 3.02 & 4.12 & & 14.52 & 20.65 \\
        64 & 2.97 & 4.02 & & 14.37 & 19.79 \\
        128 & 2.99 & 4.08 & & 14.45 & 20.12 \\
        256 & 3.05 & 4.15 & & 14.62 & 20.88 \\
        512 & 3.08 & 4.18 & & 14.75 & 21.05 \\
        \bottomrule
        \end{tabular}
    \label{table14}
\end{table}

\subsection{Imbalanced Regression}
\label{imbalance}
Medical imaging datasets often exhibit natural imbalances in target value distributions. As demonstrated in Table \ref{table7}, our SupReMix method substantially improves performance on imbalanced regression, working synergistically with established solutions like LDS and FDS \citep{pmlr-v139-yang21m}. Using the RSNA bone age dataset, which shows natural age distribution imbalances, our method improves overall MAE by approximately \textbf{15\%} and enhances few-shot performance by \textbf{30\%} compared to vanilla regression.

\begin{table}[h!]
    \centering
    \footnotesize
    \setlength{\tabcolsep}{2.3pt}
    \caption{Evaluation on imbalanced regression using RSNA dataset. Many: many-shot region (bins with $>$100 training samples), Few: few-shot region (bins with $<$20 training samples).}
    \begin{tabular}{lcccc|cccc}
    \toprule 
    Metrics & \multicolumn{4}{c}{ MAE \(\downarrow\)} & \multicolumn{4}{c}{GM \( \downarrow\)} \\
    Shot & All & Many & Med & Few & All & Many & Med & Few \\
    \hline 
    Vanilla & 4.55 & 4.12 & 5.23 & 7.85 & 45.8 & 38.2 & 52.4 & 72.8 \\
    LDS + FDS & 4.32 & 4.08 & 4.89 & 6.92 & 43.2 & 37.8 & 48.5 & 65.4 \\
    \textbf{SupReMix} + LDS & 3.96 & \textbf{3.85} & 4.52 & 6.45 & 39.8 & \textbf{35.2} & 45.8 & 62.3 \\
    \textbf{SupReMix} + FDS & 3.92 & 3.88 & 4.48 & \textbf{6.12} & 39.5 & 35.6 & 44.9 & 60.8 \\
    \textbf{SupReMix} + LDS + FDS & \textbf{3.87} & 3.86 & \textbf{4.45} & 6.18 & \textbf{38.9} & 35.4 & \textbf{44.5} & \textbf{60.2} \\
    \midrule
    GAINS (VS. Vanilla (\%)) & \textbf{\textcolor{darkergreen}{+14.9}} & \textbf{\textcolor{darkergreen}{+6.3}} & \textbf{\textcolor{darkergreen}{+14.9}} & \textbf{\textcolor{darkergreen}{+21.3}} & \textbf{\textcolor{darkergreen}{+15.1}} & \textbf{\textcolor{darkergreen}{+7.9}} & \textbf{\textcolor{darkergreen}{+15.1}} & \textbf{\textcolor{darkergreen}{+17.3}} \\
    \bottomrule
    \end{tabular}
    \label{table7}
\end{table}

\subsection{``Discretization" alternative for supervised contrastive regression}
\label{D4}
One common strategy for tackling regression using classification techniques is discretization \citep{niu2016ordinal,zhang2022improving}. We examine this approach using the Echo-Net dataset, where ejection fraction values are continuous. We vary bin sizes from 1 (original values) to 20 during the contrastive learning stage. Table \ref{table16} shows performance degradation with increased bin sizes, suggesting that discretization hampers the natural continuity of regression data.

\begin{table}[h]
    \centering
        \caption{Impact of bin size variation on Echo-Net performance}
        \begin{tabular}{ccc}
        \toprule
        Bin Size & MAE $\downarrow$ & GM $\downarrow$\\
        \hline
        1 & 4.02 & 19.79 \\
        5 & 4.18 & 20.45 \\
        10 & 4.35 & 21.88 \\
        20 & 4.62 & 23.54 \\
        \bottomrule
        \end{tabular}
    \label{table16}
\end{table}

\subsection{Comparison with C-Mixup}
We perform additional comparisons between SupReMix and C-Mixup \citep{yao2022c} across our medical imaging datasets. Results (Figure \ref{tab:mixup_comparison}) show SupReMix consistently outperforms C-Mixup trained from scratch, while combining SupReMix pretraining with C-Mixup yields further improvements.

\begin{table}[h]
    \centering
    \caption{Performance comparison of SupReMix with C-Mixup across medical datasets (MAE${\downarrow}$)}
    \begin{tabular}{lcccc}
        \toprule
        Method & UK Biobank & HCP & Echo-Net & RSNA \\
        \midrule
        C-Mixup & 3.05 & 5.92 & 4.15 & 4.28 \\
        \textbf{SupReMix} & \textbf{2.97} & 5.80 & 4.02 & 4.20 \\
        \midrule
        \textbf{SupReMix} + C-Mixup & 2.95 & \textbf{5.75} & \textbf{3.98} & \textbf{4.15} \\
        \rowcolor{white} GAINS (\textbf{Joint} VS. C-Mixup(\%)) & \textbf{\textcolor{darkergreen}{+3.3}} & \textbf{\textcolor{darkergreen}{+2.9}} & \textbf{\textcolor{darkergreen}{+4.1}} & \textbf{\textcolor{darkergreen}{+3.0}} \\
        \bottomrule
    \end{tabular}
    \label{tab:mixup_comparison}
\end{table}

\subsection{Pair Selection Strategy}
We explore different pair selection strategies using the UK Biobank dataset. Results (Figure \ref{tab:pair_selection})  show that input-level mixup significantly reduces SupReMix effectiveness, while our embedding-level approach outperforms C-Mixup.

\begin{table}[htbp]
    \centering
    \caption{Comparison of pair selection strategies on UK Biobank}
    \begin{tabular}{lccc}
        \toprule
        Strategy & MAE ${\downarrow}$ & MSE ${\downarrow}$ & GM ${\downarrow}$\\
        \midrule
        Input Mixup \citep{zhang2018mixup} & 3.25 & 15.88 & 15.23 \\
        C-Mixup \citep{yao2022c} & 3.05 & 14.82 & 14.65 \\
        SupReMix & \textbf{2.97} & \textbf{14.37} & \textbf{14.37} \\
        \bottomrule
    \end{tabular}
    \label{tab:pair_selection}
\end{table}

\subsection{Comparison with generative pre-training baselines}
We compare SupReMix against state-of-the-art generative pre-training methods on medical imaging tasks. Results (Figure \ref{tab:generative_comparison}) demonstrate the consistent superiority of our approach across modalities.

\begin{table}[htbp]
    \centering
    \caption{Comparison with generative pre-training on UK Biobank}
    \begin{tabular}{lccc}
        \toprule
        Method & MAE ${\downarrow}$ & MSE ${\downarrow}$ & GM ${\downarrow}$\\
        \midrule
        MAE (ViT-Base) \citep{he2022masked} & 3.28 & 15.92 & 15.45 \\
        MAE (ViT-Large) \citep{he2022masked} & 3.15 & 15.45 & 15.12 \\
        SupReMix & \textbf{2.97} & \textbf{14.37} & \textbf{14.37} \\
        \bottomrule
    \end{tabular}
    \label{tab:generative_comparison}
\end{table}

\newpage
\subsection{Other metrics }
In addition to Mean Absolute Error (MAE) reported in the main text, we evaluate model performance using several complementary metrics:


\begin{figure}[htbp]
    \centering
    \includegraphics[width=\columnwidth]{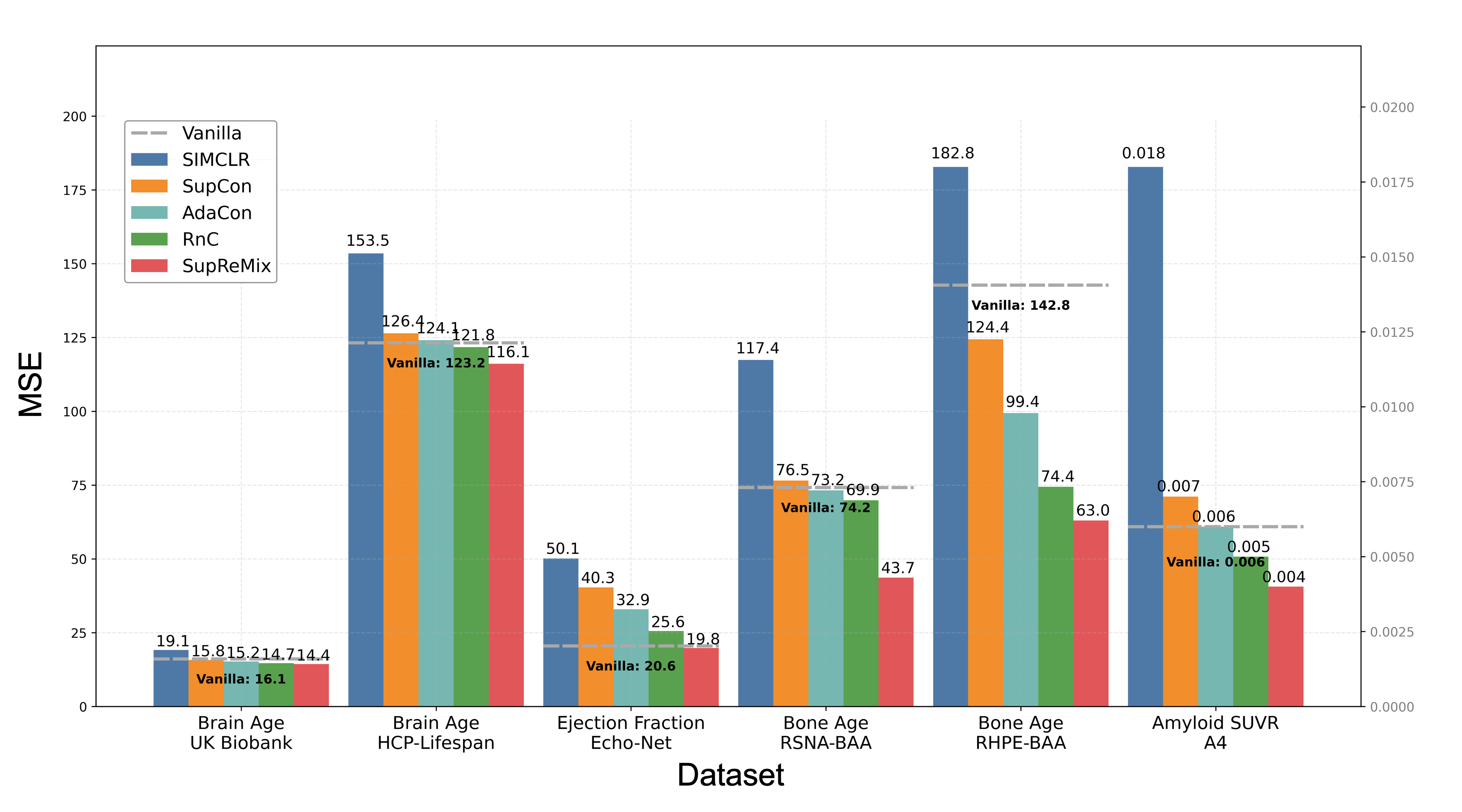}
    \caption{
    \textbf{Mean Squared Error (MSE) Comparisons Across Datasets.}}
    \label{fig:mse_comparision}
\end{figure}

\begin{figure}[htbp]
    \centering
    \includegraphics[width=\columnwidth]{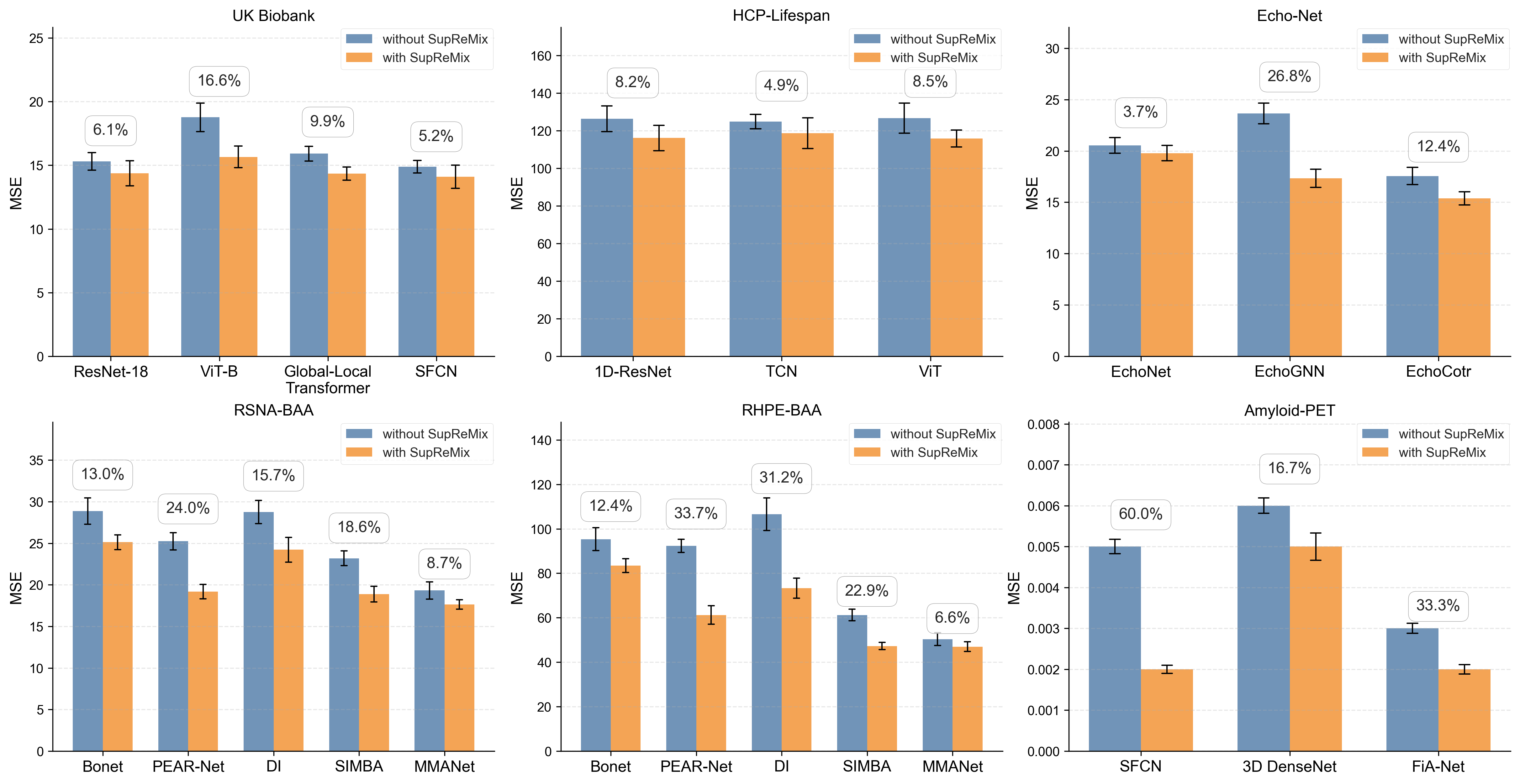}
    \caption{
    \textbf{Mean Squared Error (MSE) comparisons between task-specific methods with and without SupReMix pretraining.}}
    \label{fig:multi_panel_mse_comparison}
\end{figure}


\begin{figure}[htbp]
    \centering
    \includegraphics[width=\columnwidth]{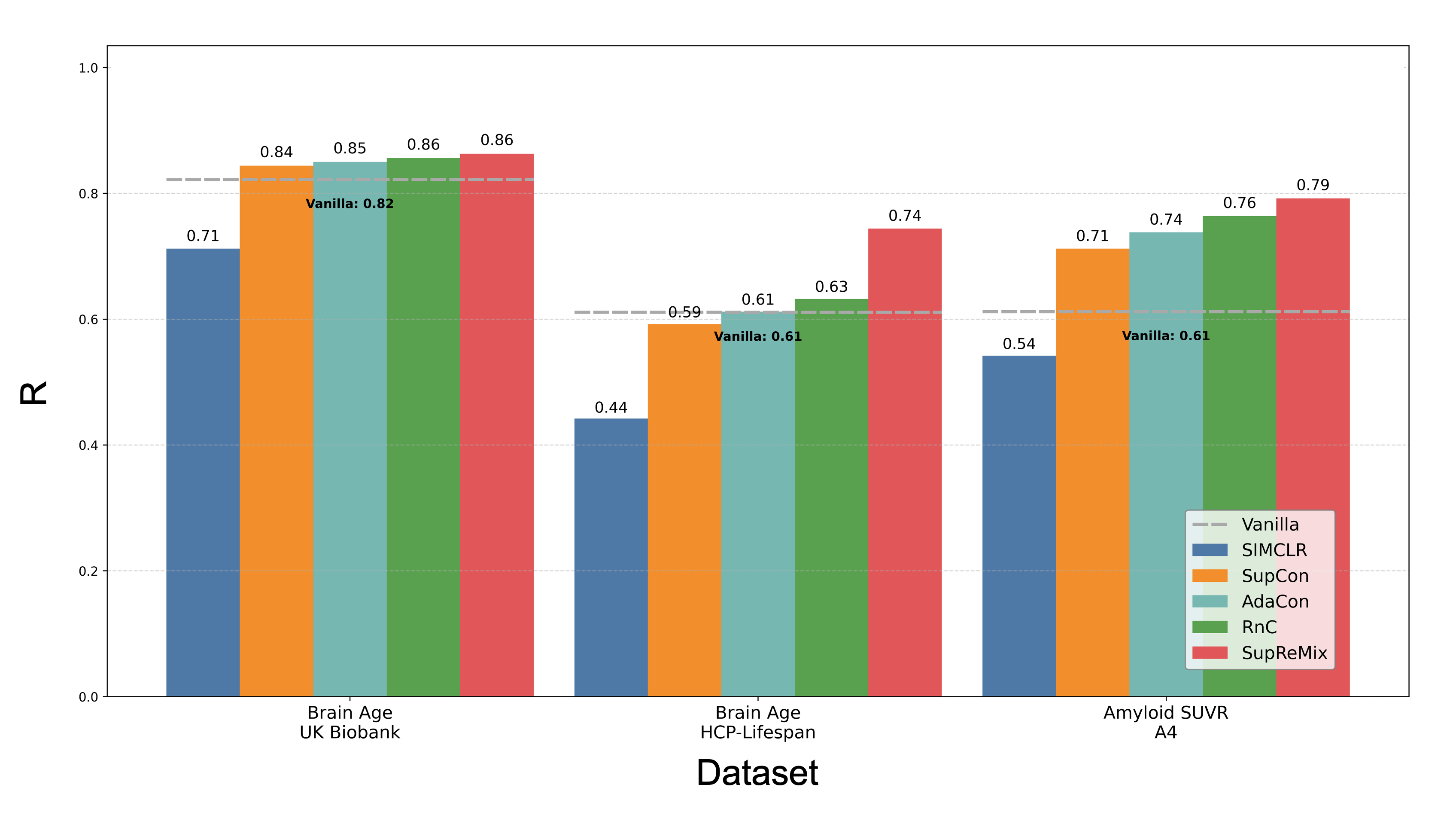}
    \caption{
    \textbf{Pearson Correlation (R) Comparisons Across Datasets.}}
    \label{fig:mse_comparision}
\end{figure}

\begin{figure}[htbp]
    \centering
    \includegraphics[width=\columnwidth]{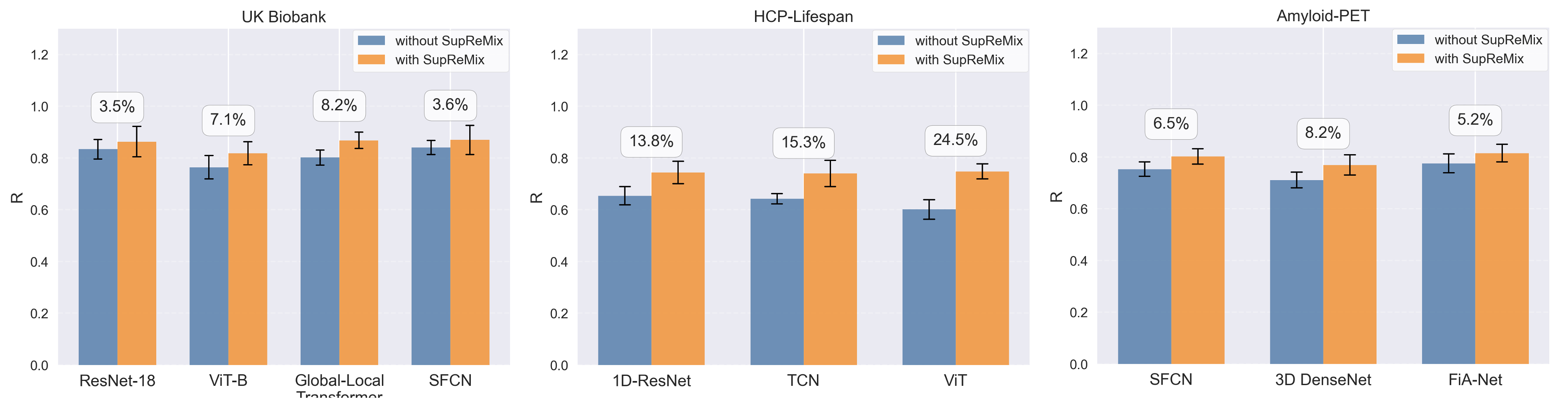}
    \caption{
    \textbf{Pearson Correlation (R) comparisons between task-specific methods with and without SupReMix pretraining.}}
    \label{fig:multi_panel_r_comparison}
\end{figure}


\begin{figure}[htbp]
    \centering
    \includegraphics[width=\columnwidth]{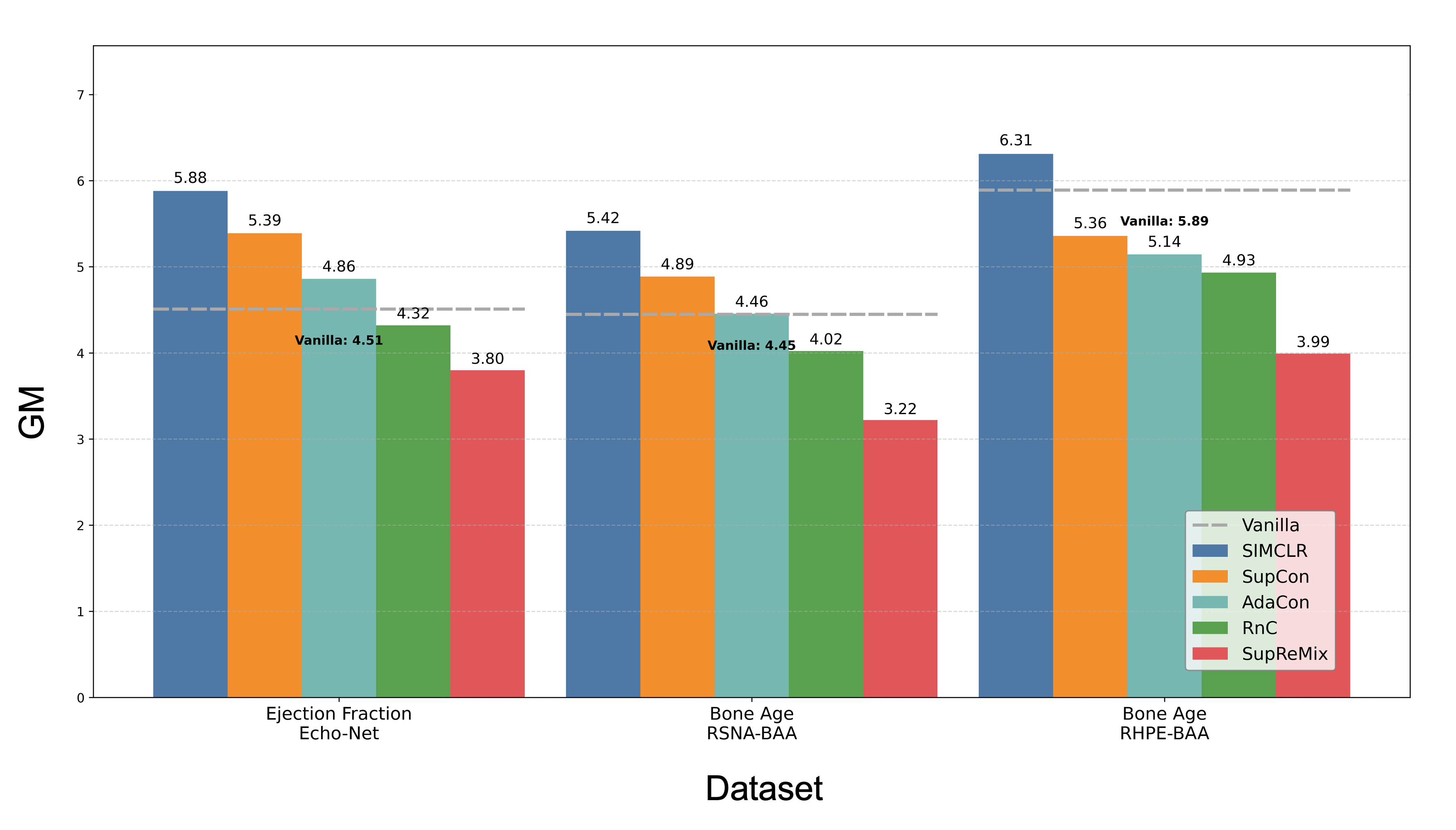}
    \caption{
    \textbf{Geometric Mean Error (GM) Comparisons Across Datasets.}}
    \label{fig:mse_comparision}
\end{figure}

\begin{figure}[htbp]
    \centering
    \includegraphics[width=\columnwidth]{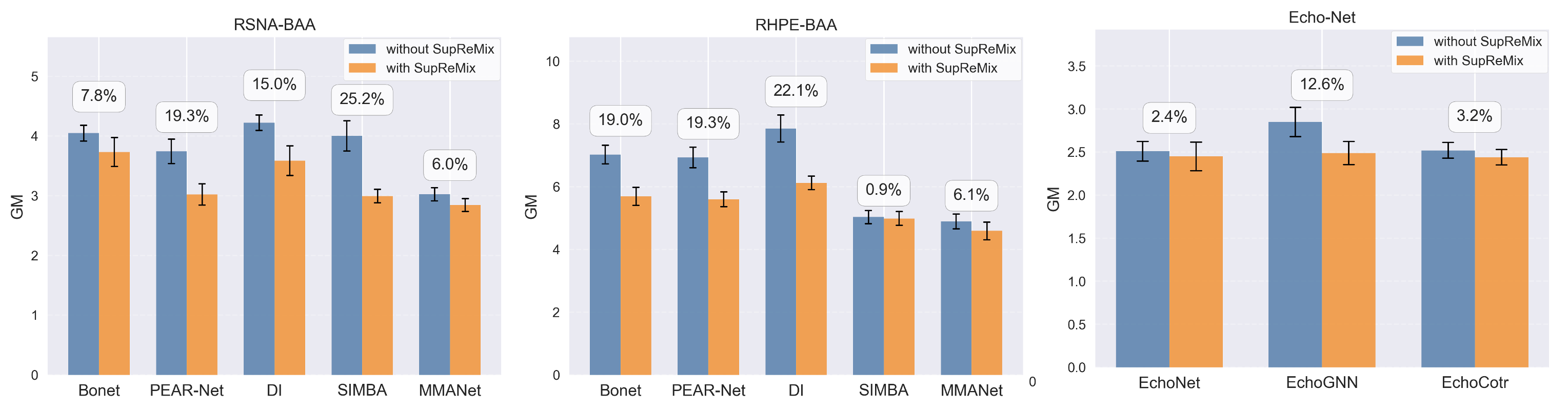}
    \caption{
    \textbf{Geometric Mean Error (GM) comparisons between task-specific methods with and without SupReMix pretraining.}}
    \label{fig:multi_panel_gm_comparison}
\end{figure}
\end{document}